\documentclass{article}
\usepackage{graphicx}
\usepackage{float}

\usepackage[main, final]{neurips_2026}
\usepackage[utf8]{inputenc} 
\usepackage[T1]{fontenc}    
\usepackage{url}            
\usepackage{booktabs}       
\usepackage{amsfonts}       
\usepackage{nicefrac}       
\usepackage{microtype}      
\usepackage{xcolor}         

\usepackage{caption}
\usepackage{subcaption}

\usepackage{amsmath}
\usepackage{amssymb}
\usepackage{mathtools}
\usepackage{amsthm}
\usepackage{thm-restate}

\usepackage[capitalize,noabbrev]{cleveref}

\newcommand{\dvol}{\,\mathrm{dvol}}
\newcommand{\ISB}{\operatorname{ISB}}

\theoremstyle{plain}
\newtheorem{theorem}{Theorem}[section]
\newtheorem{proposition}[theorem]{Proposition}

\newtheorem{lemma}[theorem]{Lemma}

\theoremstyle{definition}

\theoremstyle{remark}
\newtheorem{remark}[theorem]{Remark}

\usepackage[textsize=tiny]{todonotes}

\renewcommand{\paragraph}[1]{\smallskip\noindent{{\bf #1}}}

\title{SPHERE-JEPA: Spherical Prediction with Homogeneous Embeddings}

\author{%
  L\'eo Nicollier\\
  Université Paris-Saclay, CNRS,\\ ENS Paris-Saclay, Centre Borelli, \\
  Advanced Track and Trace \\
  \texttt{leo.nicollier@gmail.com}
  \And
  Max Dunitz\\
  Advanced Track and Trace
  \And
  Marc Pic\\
  Advanced Track and Trace
  \And
  Pablo Mus\'e\\
  Université Paris-Saclay, CNRS,\\ ENS Paris-Saclay, Centre Borelli \\
  \And
  Enric Meinhardt-Llopis\\
  Université Paris-Saclay, CNRS,\\ ENS Paris-Saclay, Centre Borelli \\
  \And
  Gabriele Facciolo\\
  Université Paris-Saclay, CNRS,\\ ENS Paris-Saclay, Centre Borelli, \\
  Institut Universitaire de France\\
}

\begin{document}
\maketitle
\begin{abstract}
A fundamental open question in self-supervised learning (SSL) is the explicit characterization of the optimal geometry of the learned representations.
Recently, LeJEPA identified isotropic Gaussian embeddings as optimal for minimizing downstream prediction risk in Euclidean spaces. 
However, the corresponding problem for distributions supported on lower-dimensional manifolds, such as the hypersphere, remains unexplored.
In this work, we demonstrate that extending this minimax analysis to smooth distributions on Riemannian manifolds fundamentally changes the optimal solution. 
We show that, under a worst-case formulation, both \(k\)-nearest neighbors and kernel ridge regression induce hyperspherical uniformity.
More precisely, we show that uniform distributions on manifolds are optimal for \(k\)-nearest neighbors, and that the uniform distribution on the sphere is optimal for kernel ridge regression with both the exponential dot-product kernel and the linear kernel. 
This theoretical insight reveals a fundamental limitation of Gaussian embeddings: their non-uniform density induces anisotropic $k$-NN  neighborhoods, severely biasing the estimator. 
To correct this, we introduce \textbf{SPHERE-JEPA}, a theoretically grounded SSL framework. 
We adapt LeJEPA's Cramér--Wold projection mechanism to enforce hyperspherical uniformity rather than a Gaussian prior. 
Empirically, SPHERE-JEPA achieves competitive results on standard benchmarks, including a $+1.8$ percentage-point linear-probing gain on ImageNet-1K (ViT-B/14), and improves texture-retrieval mAP by
over six percentage points in a controlled nonparametric evaluation.
\end{abstract}

\section{Introduction}
Self-supervised learning (SSL) has emerged as a powerful paradigm for learning visual representations without manual annotations~\citep{simclr}. 
By leveraging invariances across multiple augmented views of the same data, modern SSL methods such as SimCLR~\citep{simclr}, BYOL~\citep{byol}, and DINO~\citep{dinov1} have achieved performance comparable to supervised learning across a wide range of tasks.

Until recently, most SSL methods relied primarily on heuristic mechanisms to prevent representation collapse and shape the geometry of learned embeddings, without explicitly characterizing what constitutes an optimal representation space.
With 
LeJEPA,~\citet{balestriero2025lejepaprovablescalableselfsupervised} provided a first step towards such a characterization by showing that an isotropic Gaussian distribution is optimal for linear ridge regression and $k$-nearest neighbor regression in a worst-case formulation. 
These two evaluation protocols are commonly used to assess the quality of SSL representation, as they probe, respectively, the linear separability of features and the local geometric structure of the embedding space.
However, this analysis is restricted to distributions with densities in $\mathbb{R}^d$, and therefore excludes a broad class of natural representations supported on lower-dimensional manifolds, such as the hypersphere.

\begin{figure}[t]
  \centering
  \includegraphics[width=0.8\linewidth]{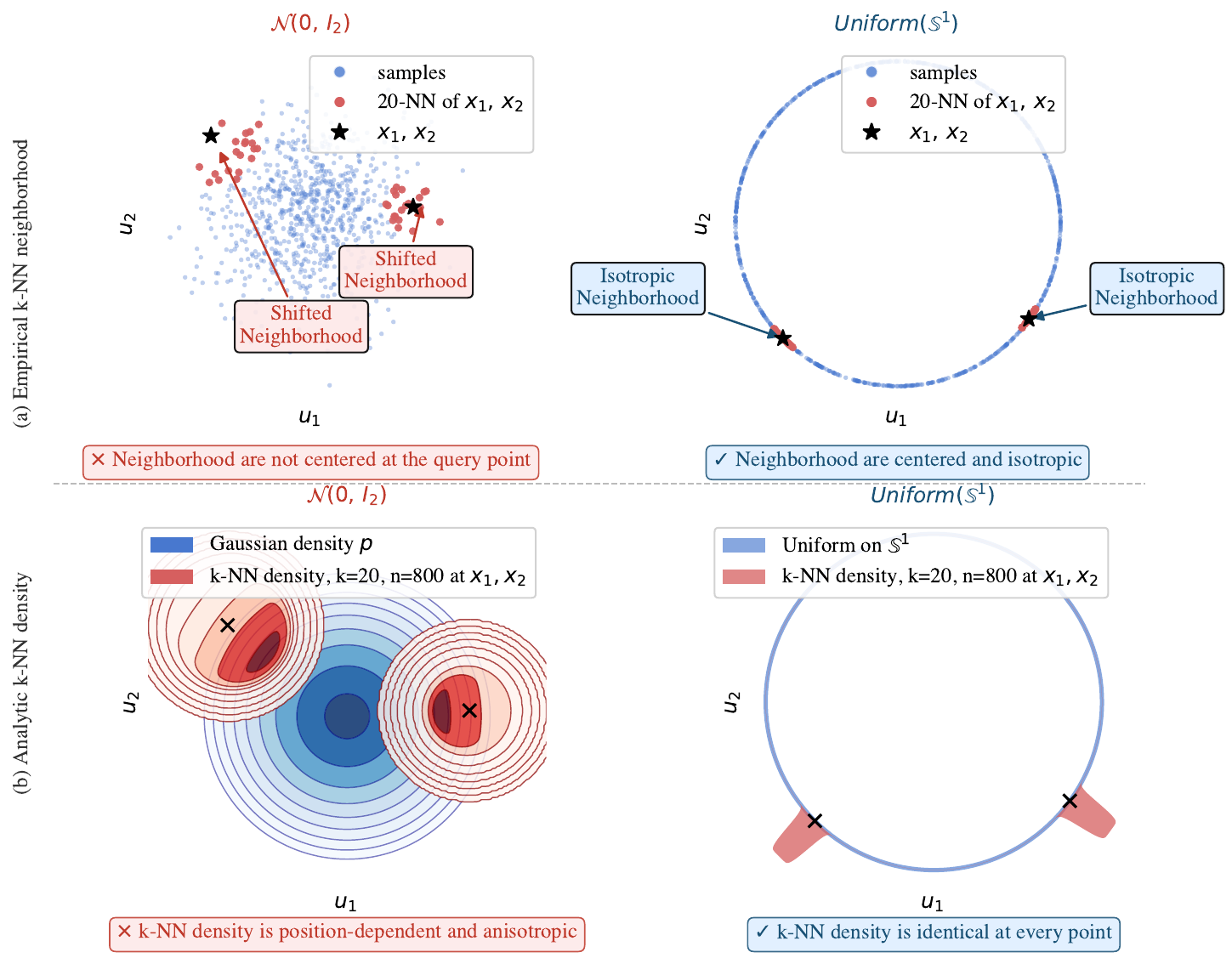}
  \caption{%
    \textbf{k-NN  
    neighborhoods are density-biased.}
    \textbf{(a)} Empirical $k$-NN neighborhoods ($k=20$) at two query points $x_1, x_2$: under a Gaussian distribution on $\mathbb{R}^2$, neighborhoods are not centered at the query point but are skewed toward regions of higher density, resulting in anisotropic and directionally biased neighborhoods.
    In contrast, for a uniform distribution on $\mathbb{S}^1$ the neighborhoods are centered and isotropic everywhere.
    \textbf{(b)} Analytic $k$-NN density: under a Gaussian distribution, the induced densities are position-dependent and anisotropic, whereas the uniform distribution on $\mathcal{S}^1$ yields identical densities at every point, reflecting its intrinsic geometric symmetry.}
  \label{fig:fig1}
\end{figure}

In this work, we extend this analysis to smooth distributions supported on Riemannian manifolds. Under the same worst-case formulation as LeJEPA, we derive geometric optimality principles for both local and kernel-based learning methods.

First, we show that minimizing the worst-case bias of \(k\)-nearest neighbors implies that the representation distribution must be uniform on the underlying manifold.
Second, we prove that minimizing the worst-case bias of kernel ridge regression with both the linear kernel and the exponential dot-product kernel \(K(x,y)=e^{\kappa x^\top y}\) forces the representation distribution to be uniform on the hypersphere.

This result contrasts with Gaussian representations, whose non-uniform density induces anisotropic local neighborhoods, thereby distorting the local geometry and degrading the performance of nonparametric methods such as $k$-NN (see Figure~\ref{fig:fig1}).
Building on this theoretical insight, we introduce \textbf{SPHERE-JEPA}, a self-supervised learning framework that enforces hyperspherical uniformity.
Our approach follows the same design as LeJEPA and relies on the same projection-based regularization mechanism inspired by the Cramér--Wold theorem~\citep{CuestaAlbertos2007CramerWold}, which enforces a target distribution through random one-dimensional projections. 
We repurpose this projection-based mechanism to enforce a \emph{uniform distribution on the hypersphere} rather than a Gaussian distribution in $\mathbb{R}^d$, leading to \textbf{Sketched Uniform Spherical Regularization (SUSReg)}.

We evaluate our method on ImageNet-100~\citep{imagenet} using a ResNet-18, and on ImageNet-1K using a ViT-B/14~\citep{dinov2} backbone. 
Our results show broadly comparable performance on standard benchmarks, with a larger gain on a controlled texture-retrieval task designed to probe $k$-NN behavior in the learned representation space.
Notably, on ImageNet-1K, it yields a $+1.8\%$ accuracy gain in linear probing and improves average transfer performance across diverse downstream datasets. 
Furthermore, in a controlled texture retrieval task designed to probe $k$-NN behavior, our approach yields significant gains, improving mean average precision by over 6\% in the learned representation space, without incurring any additional computational cost.

\paragraph{Contributions.}
Our key contributions are as follows: 
{\em (i)} We extend the theoretical analysis of optimal SSL representations from Euclidean densities to distributions supported on smooth manifolds; 
{\em (ii)} We prove that the uniform distribution on the hypersphere is optimal under a worst-case formulation that combines radial kernel ridge regression and $k$-NN; 
{\em (iii)} We introduce SUSReg, a projection-based regularizer that promotes hyperspherical uniformity in learned representations; 
{\em (iv)} We report gains on ViT-B/14 linear probing and controlled nonparametric retrieval.

\paragraph{Organization.}
The remainder of the paper is structured as follows. Section~\ref{sec:notation} introduces the notation and setup, followed by related work in Section~\ref{sec:related}. 
Section~\ref{OptimalityRepresentation} develops the theoretical analysis of optimal representation geometries for \(k\)-nearest neighbors and kernel ridge regression.
In Section~\ref{sec:method}, we introduce SPHERE-JEPA and the SUSReg regularization. 
Experiments are reported in Section~\ref{sec:experiments}.
We conclude in Section~\ref{conclusion}.

\section{Notation and Setup} \label{sec:notation}
We adopt the same notation as in
\cite{balestriero2025lejepaprovablescalableselfsupervised}.
We consider a dataset composed of $N$ independent samples. Each sample is observed through $V$ views, yielding data points
$x_{n,v}\in\mathbb R^{D}$, $n=1,\dots,N$, $v=1,\dots,V$,
where $D$ denotes the input dimension (e.g., for an image of spatial resolution $H \times W$ with $C$ channels, $D = C \times H \times W$).
The views typically correspond to data augmentations of a given image (e.g.,\ image crops or geometric transformations).

Following standard SSL practice~\citep{SwAV, dinov1}, we distinguish between \textit{local} and \textit{global} views~\citep{caron2021emerging}. Both are obtained as crops from the input image: local views correspond to smaller crops that capture limited spatial context, while global views are larger crops that preserve most of the image content.
We denote by $V_g$ the number of global views, by $V_l$ the number of local views, and by $V_a$ the total number
of views.
We index all views $v=1,\ldots, V_a$, whith the first $V_g$ indices corresponding to global views.
%
We assume that the samples $\{x_n\}_{n=1}^N$ are independent and identically distributed.

\paragraph{Encoder and projection head.}
Let $g_\phi:\mathbb R^D\to\mathbb R^h$ denote the backbone encoder and $h_\psi:\mathbb R^h\to\mathbb R^d$ the MLP projection head.
We write $f_\theta=h_\psi\circ g_\phi$, with $\theta=(\phi,\psi)$.
The backbone architecture may be convolutional or based on self-attention.
For each sample and view, the projected embedding is $z_{n,v}\coloneqq f_\theta(x_{n,v})\in\mathbb R^d$.
We denote its normalized version by
\begin{equation}
\tilde z_{n,v}
\coloneqq \frac{z_{n,v}}{\|z_{n,v}\|}
\in\mathbb S^{d-1}.
\end{equation}

\section{Related Work} \label{sec:related}

\subsection{Self-Supervised Learning}
Self-supervised learning (SSL)~\citep{simclr} aims to learn transferable representations without manual annotations by leveraging the intrinsic structure and invariances of the data. This is typically achieved by enforcing consistency between representations of different augmented views of the same sample, while separating those of distinct samples.
A fundamental challenge shared by most SSL objectives~\citep{byol} is to avoid representation collapse, in which all inputs are mapped to a constant embedding. Existing methods address this issue through a variety of architectural and optimization strategies.

Contrastive methods, such as SimCLR~\citep{simclr} and MoCo~\citep{mocov1}, rely on negative samples to enforce discriminative representations. Momentum-based frameworks, such as MoCo or DINO~\citep{dinov1}, further stabilize the training with an exponential moving average (EMA) teacher.
Non-contrastive approaches, including BYOL~\citep{byol} and SimSiam~\citep{simsiam}, eliminate the need for negative samples but introduce explicit asymmetries during training, such as stop-gradient operations or additional predictor networks.
Decorrelation-based approaches, such as Barlow Twins~\citep{barlowtwin} and VICReg~\citep{vicreg}, prevent collapse by explicitly penalizing correlations across embedding dimensions.
Beyond the choice of the objective, DINO shows that architectural design also plays a critical role in training stability. In particular, the use of register tokens in Vision Transformers helps maintain feature diversity and mitigates collapse by allocating dedicated capacity for global information~\citep{dinov2, dinov3}.

Although these strategies have demonstrated strong empirical performance, they rely primarily on heuristic mechanisms to prevent collapse.
In most cases, the structure of the optimal representation is implicitly defined by training dynamics or data augmentations rather than explicitly characterized. 
As a result, many SSL objectives enforce invariances that are sufficient for good performance, but whose optimality properties remain only partially understood.

\subsection{JEPA and LeJEPA}
Joint Embedding Predictive Architectures (JEPAs), introduced by \cite{JepaLecun}, proposes a representation learning framework based on predictability in the latent space. 
Rather than relying on pixel-level matching or contrastive objectives, a JEPA learns representations by predicting the latent representation of one view (or future state) from another, thereby emphasizing the structure and dynamics of the representation itself. 
In this sense, JEPAs can be viewed as a generalization of earlier non-contrastive methods such as BYOL~\citep{byol} and SimSiam~\citep{simsiam}, which also rely on direct prediction in a latent space.

LeJEPA~\citep{balestriero2025lejepaprovablescalableselfsupervised} instantiates the JEPA framework by retaining its core latent prediction objective. 
Like BYOL or SimSiam, it learns representations by predicting the embedding of one view from another in the latent space. 
Overall, LeJEPA is composed of two complementary terms:
{\em(i)} a predictive loss, similar to BYOL/SimSiam, implemented as a mean squared error between representations, which promotes alignment across views; {\em (ii)} a distributional regularization term that constrains the geometry of the representation space.

Specifically, this regularization encourages the learned embeddings to match a prescribed isotropic Gaussian distribution.
%
This constraint is implemented via random one-dimensional projections of the embeddings, leveraging the Cramér--Wold theorem~\citep{CuestaAlbertos2007CramerWold}. 
For each view $v\in\{1,\ldots,V_a\}$, consider a mini-batch of embeddings $\{ z_{n,v}\}_{n=1}^N\subseteq\mathbb R^{d}$. 
At each training step, LeJEPA draws a finite set of directions $\mathcal A=\{a_1,\dots, a_{|\mathcal{A}|}\}$ with $a\sim \mathrm{Unif}(\mathbb S^{d-1})$, and forms the sliced variables
$t_{n,v}^{a} \coloneqq a^\top  z_{n,v}\in\mathbb R$.
It then enforces, for each view $v$ and each direction $a \in \mathcal{A}$ (with $\mathcal A$ resampled at every mini-batch), that the empirical distribution of $\{t_{n,v}^{a}\}_{n=1}^N$ matches a standard normal $\mathcal{N}(0,1)$ via the Epps--Pulley test. 
By the Cramér--Wold theorem, matching the one-dimensional projections to $\mathcal N(0,1)$ across uniform directions enforces an isotropic Gaussian embedding distribution in $\mathbb R^d$.

\subsection{Epps--Pulley test}\label{subsec_EP}

LeJEPA uses the test by~\citet{EPTest} (EP) as a univariate statistical discrepancy based on characteristic functions. 
Given a scalar random variable $X$ with samples $\{x_j\}_{j=1}^N$ and a reference distribution $Y$ with characteristic function $\varphi_Y(t)=\mathbb E[e^{itY}]$, the empirical characteristic function of $X$ is $\hat{\varphi}_X(t)=\frac{1}{N}\sum_{j=1}^N e^{it x_j}$. 
The EP discrepancy is defined as a weighted $L^2$ distance between characteristic functions:
\begin{equation}
\mathrm{EP}(X,Y) = N\int_{\mathbb R}\bigl|\hat{\varphi}_X(t)-\varphi_Y(t)\bigr|^2\,w(t)\,dt,
\end{equation}
where $w(t)$ is a Gaussian weight function. 
In LeJEPA, this test is applied to the projected scalars $x_j=t_{j,v}^a=a^\top z_{j,v}$, with $Y\sim\mathcal N(0,1)$, and averaged over a finite set of directions $a\in\mathcal A$.

\subsection{Distributional Regularization of Latent Representations}

In addition to LeJEPA, other recent self-supervised learning methods incorporate explicit regularization of the latent space to prevent collapse and shape the geometry of the learned representations~\citep{koleo,vicreg,dinov1}.
These regularizers can be grouped according to the target distribution they enforce on the embeddings.

A first group of methods targets an isotropic Gaussian distribution in $\mathbb{R}^d$. 
VICReg~\citep{vicreg} enforces this only through low-order statistics: it maintains non-zero variance in each coordinate and decorrelates coordinates across the batch, promoting isotropy at the covariance level.
LeJEPA enforces a stronger, distribution-level constraint by matching the embedding distribution to an isotropic Gaussian via random one-dimensional projections and the Epps--Pulley test.

Normalized spherical embeddings are also used in contrastive methods such as SimCLR~\citep{simclr} and MoCo~\citep{mocov1}; spherical representations are therefore not specific to SPHERE-JEPA.
Normalization alone, however, does not enforce a uniform distribution on the sphere.
KoLeo~\citep{koleo}, notably employed in DINO, encourages more uniform coverage through a mini-batch estimate of differential entropy.
Its nearest-neighbor and logarithmic terms can be sensitive to nearly coincident embeddings, motivating numerical safeguards in implementations.

The choice of a regularizer poses a fundamental question regarding downstream performance: 
Do different distributional constraints favor specific estimators, such as linear models, kernel methods, or $k$-nearest neighbors?

\section{Optimality Criteria for Self-Supervised Representations}
\label{OptimalityRepresentation}

A central challenge in self-supervised learning (SSL) is to learn representations that transfer efficiently to a broad range of downstream
tasks with minimal supervision~\citep{simclr}.
Since downstream tasks are unknown at training time, optimizing representations for a single task may lead to poor transferability. 
Following LeJEPA, we adopt a \emph{minimax} perspective: we seek a representation that minimizes the worst-case prediction error over a class of downstream targets.

Representation quality is commonly assessed through two classical downstream protocols: linear ridge regression~\citep{scholkopf2002learning}
and \(k\)-nearest neighbors (\(k\)-NN)~\citep{hastie2009elements}.

\paragraph{Linear ridge regression.}
This criterion probes the \emph{global structure} of the representation space. Given a dataset of representations and targets \(\{(z_i,y_i)\}_{i=1}^N\), it estimates a linear predictor \(g_w(z)=\langle w,z\rangle\) by minimizing the regularized empirical risk
\[
\widehat w_\lambda
=
\arg\min_{w\in\mathbb R^d}
\frac1N
\sum_{i=1}^N
(y_i-\langle w,z_i\rangle)^2
+
\lambda\|w\|^2,
\qquad
\lambda>0.
\]

\paragraph{\(k\)-nearest neighbors.}
Conversely, $k$-NN is sensitive to the \emph{local geometry}, as its prediction bias depends heavily on the neighborhood shape and the local sampling density of the data. Together, they capture complementary aspects of representation quality.

\paragraph{Minimax formulation.}
Let \(p\) denote the marginal distribution of the representations
\(z=f_\theta(x)\in\mathbb R^d\). Given a downstream estimator
\(\widehat g\) trained on samples drawn from \(p\), we measure its
performance through the integrated squared bias (ISB)
\[
\mathrm{ISB}(g,p)
=
\int
\bigl(
\mathbb E[\widehat g(z)]-g(z)
\bigr)^2
\,p(z)\,\mathrm{dvol}(z).
\]

Given a class \(\mathcal G\) of downstream targets, we consider the
minimax problem
\[
p^\star
=
\arg\min_p
\sup_{g\in\mathcal G}
\mathrm{ISB}(g,p).
\]

Different downstream protocols induce different geometric constraints on
the optimal representation distribution.
\begin{proposition}[Linear ridge regression, \citep{balestriero2025lejepaprovablescalableselfsupervised}]
\label{prop:lrr}
Consider the linear ridge regression with a downstream target class
\[
\mathcal G
=
\left\{
g_w(x)=\langle w,x\rangle
\;:\;
\|w\|\le1
\right\}.
\]

Under the constraint \(\mathbb E[\|X\|^2]=1\), minimizing the worst-case ISB is equivalent to minimizing the top eigenvalue of the covariance matrix of \(p\). Consequently, any optimal
distribution is isotropic.
\end{proposition}

\begin{proposition}[$k$-NN on compact and Euclidean supports]
\label{prop:knn}
Let $c>0$, $k\to\infty$, and $k/n\to0$.
Let $M_n\to\infty$ satisfy $M_n(k/n)^{1/m}\to0$, and set $|f|_{C^3}=\max_{1\le j\le3}\|\nabla^j f\|_\infty$.
The ISB below refers to the population-radius estimator.
If $\mathcal M\subset\mathbb R^d$ is a compact, connected, $m$-dimensional Riemannian submanifold without boundary, the uniform
density on $\mathcal M$ is the unique pointwise asymptotic minimizer of the worst-case ISB over
\[
\mathcal G_n^{\mathcal M}
=
\left\{f\in C^3(\mathcal M):
\|\Delta_{\mathcal M}f\|_2\le c,\ 
|f|_{C^3}\le M_n\right\}.
\]

If $\mathcal M=\mathbb R^m$, then for every positive smooth probability density $p$, the worst-case leading-order ISB functional over the class obtained by setting $\mathcal M=\mathbb R^m$
in the definition above diverges as $n\to\infty$. 
In particular, this holds for Gaussian densities.
\end{proposition}

The previous two criteria are not sufficient to fully constrain the
optimal geometry of the representation distribution. We therefore
introduce a third minimax criterion based on kernel ridge regression~\citep{scholkopf2002learning}
(KRR) with the exponential dot-product kernel $K(x,y)=e^{\kappa x^\top y}$ with  $\kappa>0$.

Let \(T_p\) denote the associated population covariance operator (see Appendix~\ref{app:worstcasekrr}) and consider the source class
\[
\mathcal G_p
=
\left\{
g=T_pf
\;:\;
\|f\|_{L^2(p)}\le1
\right\}.
\]

\begin{proposition}[Exponential KRR]
\label{prop:expkrr}
For kernel ridge regression with kernel $K(x,y)=e^{\kappa x^\top y}$,
the worst-case ISB over \(\mathcal G_p\) is minimized by
$p^\star
=
\mathrm{Unif}(\mathbb S^{d-1})$.
\end{proposition}

Across the admissible distributions on Euclidean space and compact manifolds, the three minimax criteria single out the uniform distribution on the hypersphere as their unique common minimizer:

\begin{theorem}[Optimal representation geometry]
\label{thm:optimality}
The uniform distribution on the hypersphere $\mathrm{Unif}(\mathbb S^{d-1})$ simultaneously minimizes the worst-case ISB for linear ridge regression,
\(k\)-NN regression, and exponential-kernel ridge regression.
\end{theorem}

Our analysis considers smooth densities on Euclidean space and on compact manifolds, including spheres and tori. For $k$-NN, the worst-case leading functional is unbounded for positive smooth densities on $\mathbb R^d$, while the uniform density minimizes it on each fixed compact manifold. 
The uniform distribution on the sphere also minimizes the linear-ridge and exponential-KRR criteria.
It is therefore a common minimizer of the three criteria among the distributions studied here.

Our optimality guarantee concerns the three downstream estimators and target classes analyzed above.
For representations on the sphere, the KRR argument also shows that the uniform distribution is a minimizer for any positive-definite rotationally invariant kernel; uniqueness need not hold for every kernel.
This includes an MLP probe only when its training is described by regularized kernel regression with a fixed, rotationally invariant neural tangent kernel.
The result does not cover finite-width MLP feature learning or end-to-end fine-tuning, where the representation geometry can change during training.

\section{SPHERE-JEPA: Spherical Prediction with Homogeneous Embeddings} \label{sec:method}

\begin{figure}[t]
  \centering
  \includegraphics[width=0.8\linewidth]{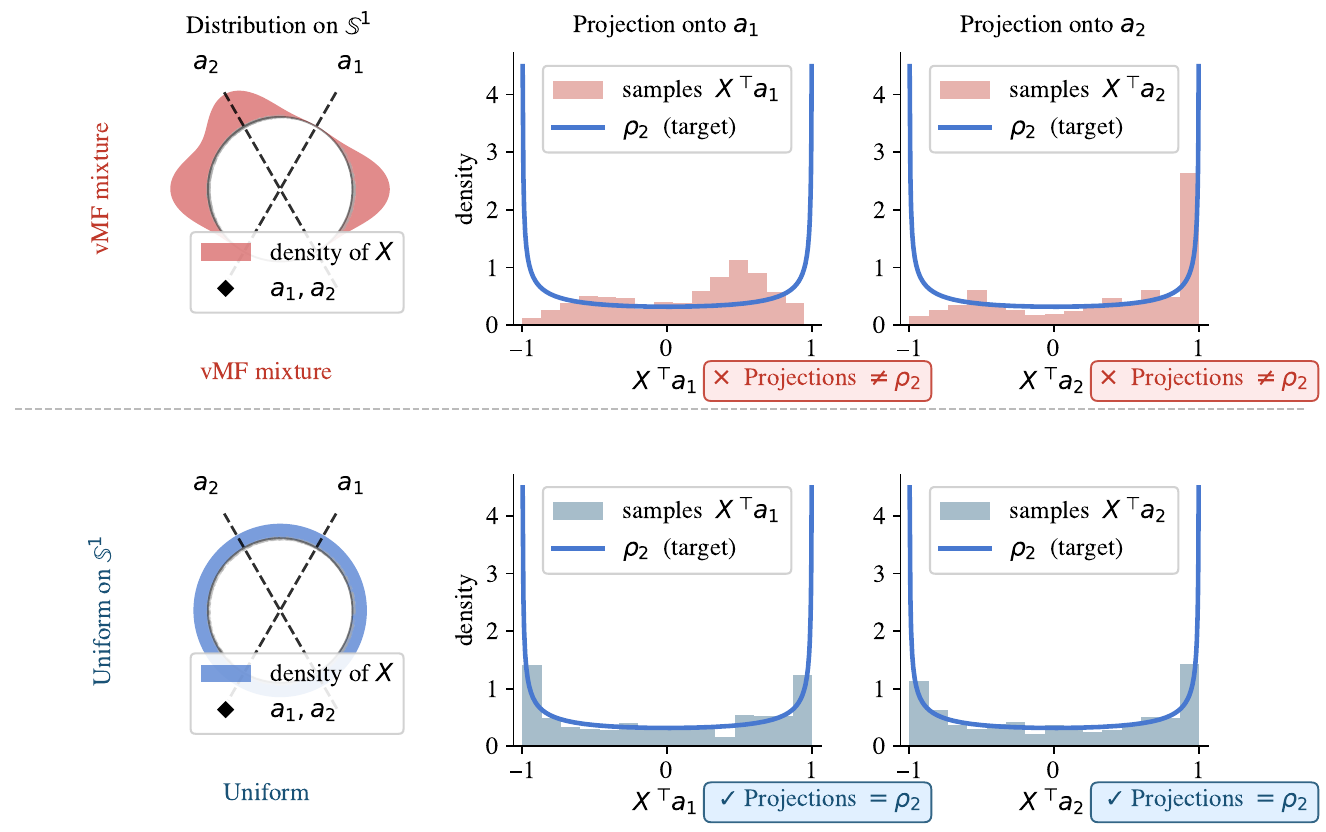}
\caption{
Illustration of the Cramér--Wold characterization underlying SUSReg on $\mathbb{S}^1$.
\textbf{Top row:} a mixture of von Mises--Fisher (vMF) distributions—the spherical analogue of Gaussian distributions—induces a non-uniform distribution with mass concentrated around a few directions. The resulting projections $X^\top a_1$ and $X^\top a_2$ deviate from the target density $\rho_2$.
\textbf{Bottom row:} a uniform distribution on $\mathbb{S}^1$ yields projections that match $\rho_2$ across directions.
Histograms compare the distributions of $a^\top X$ to the target density $\rho_2$ (solid curve). Non-uniform representations (top) produce mismatched projections and are penalized by SUSReg, whereas uniform representations (bottom) satisfy the projection constraint and are therefore not penalized by SUSReg.
}
  \label{fig:fig2}
\end{figure}

\subsection{Sketched Uniform Spherical Regularization (SUSReg)}

We introduce \emph{Sketched Uniform Spherical Regularization} (SUSReg), a projection-based regularizer inspired by SIGReg~\citep{balestriero2025lejepaprovablescalableselfsupervised}.
Let $\{\tilde z_{n,v}\}_{n=1}^N \subseteq \mathbb{S}^{d-1}$ be a normalized minibatch
for view $v$, and let $X_v$ ($X$ when the context is clear) be uniform over this set.
By the Cramér-Wold theorem, $X \sim \mathrm{Unif}(\mathbb{S}^{d-1})$ if and only if every projection $a^\top X$, $a\in\mathbb{S}^{d-1}$, follows the density 
\begin{equation}
  \rho_d(t) = C_d\,(1 - t^2)^{\frac{d-3}{2}}, \qquad t\in(-1,1),
  \label{eq:rhod}
\end{equation}
where $C_d$ is a normalization constant that ensures that $\rho_d$ integrates into one.
For large $d$, $\rho_d$ converges to the density of $\mathcal{N}(0,\frac{1}{d})$ by the central limit theorem~\citep{billingsley1995probability}. We use this Gaussian approximation when $d>256$.
This characterization is illustrated in Figures~\ref{fig:fig2}--\ref{fig:fig3}, which provide an empirical visualization of the Cramér--Wold theorem: a distribution is uniform on the sphere if and only if its one-dimensional projections match $\rho_d$ along all directions. Accordingly, only uniformly distributed representations yield projection distributions consistent with $\rho_d$, while non-uniform distributions lead to systematic, direction-dependent mismatches.

Following SIGReg, SUSReg enforces this condition by penalizing, via the Epps--Pulley test, the mismatch between $a^\top X$ and $Y \sim \rho_d$ over randomly sampled directions:
\begin{equation}
  \mathcal{L}_{\mathrm{SUSReg}}
  \;\coloneqq\;
  \frac{1}{V_a}\sum_{v=1}^{V_a}
  \frac{1}{|\mathcal{A}|}\sum_{a\in\mathcal{A}}
  \mathrm{EP}\!\left(a^\top X,\; Y\right).
  \label{eq:susreg}
\end{equation}
Minimizing this loss enforces that the projected variables $a^\top X$ match the target distribution $\rho_d$ across random directions, thereby enforcing uniformity on the hypersphere and aligning the representation geometry with the optimal structure characterized in Section~\ref{OptimalityRepresentation}.

\begin{figure}[t]
  \centering
  \includegraphics[width=0.8\linewidth]{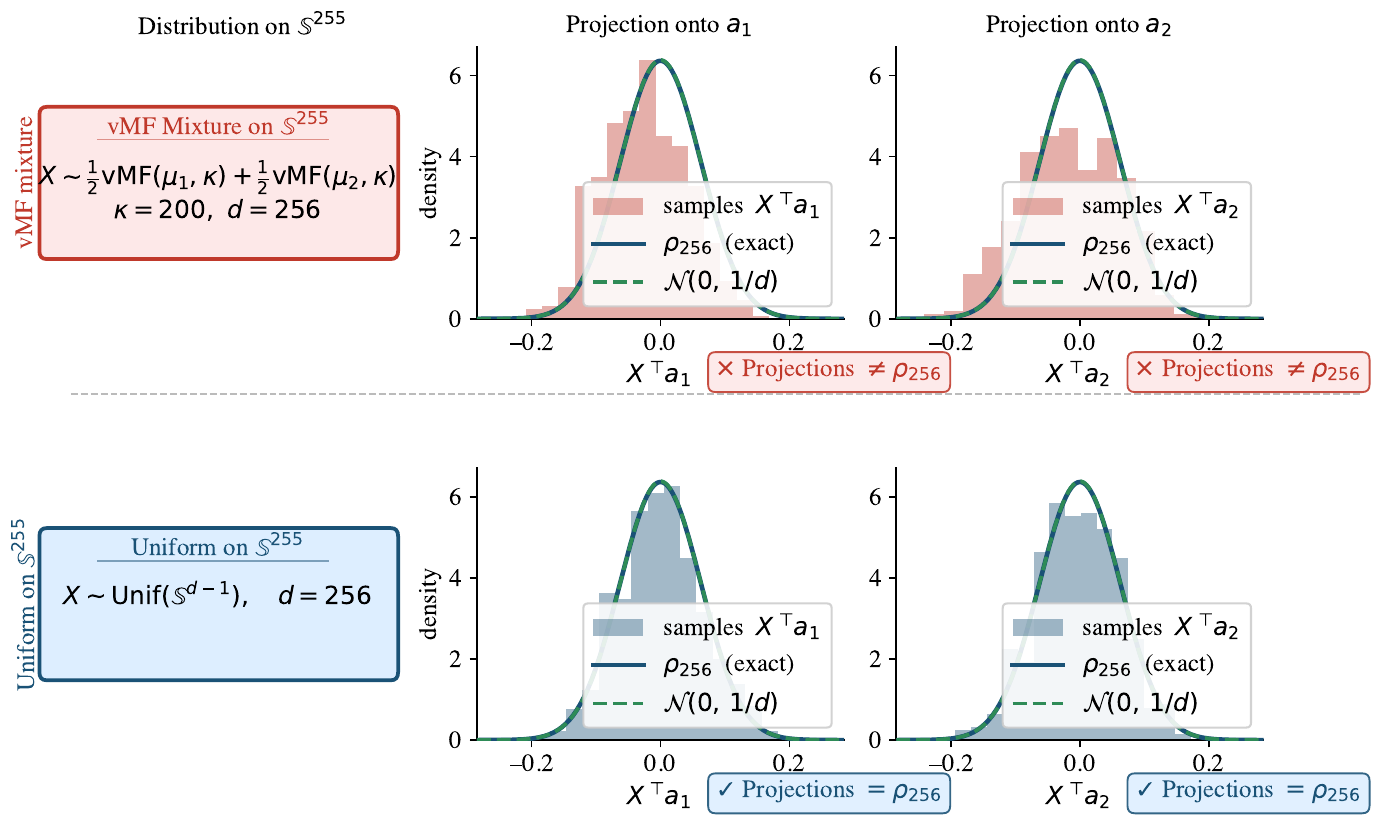}
\caption{
Illustration of the Cramér--Wold characterization underlying SUSReg on $\mathbb{S}^{255}$ ($d=256$).
\textbf{Top row:} a mixture of von Mises--Fisher (vMF) distributions induces a non-uniform distribution with mass concentrated around a few directions. The resulting projections $X^\top a_1$ and $X^\top a_2$ induce distributions that deviate from the target density $\rho_{256}$.
\textbf{Bottom row:} a uniform distribution on $\mathbb{S}^{255}$ yields projections consistent with $\rho_{256}$ across directions.
Histograms compare the distributions of $a^\top X$ to $\rho_{256}$ (solid curve). Non-uniform representations (top) produce mismatched projections and are penalized by SUSReg, whereas uniform representations (bottom) satisfy the projection constraint and are therefore not penalized by SUSReg.
The Gaussian approximation $\mathcal{N}(0, 1/d)$ (dashed) closely matches $\rho_{256}$, justifying its use in high dimensions.}
\label{fig:fig3}
\end{figure}

\subsection{Invariant Prediction Loss on the Hypersphere}

While SUSReg enforces global uniformity on the hypersphere, it does not guarantee invariance to semantic-preserving transformations. To address this, we introduce an alignment objective. Unlike LeJEPA, which applies a mean squared error (MSE) to unnormalized embeddings in $\mathbb{R}^d$, we adapt this principle to normalized representations.

For a sample with multiple views, we define a prototype $\mu_n$ as the average of its global views and align all views to it via the invariance loss $\mathcal{L}_{\mathrm{inv}}$:
\begin{equation} \label{equ:prototype}
\mu_n \;\coloneqq\; \frac{1}{V_g}\sum_{v=1}^{V_g}\tilde z_{n,v}, \qquad \mathcal L_{\mathrm{inv}} \;\coloneqq\; \frac{1}{V_a}\sum_{v=1}^{V_a} \left\|\mu_n - \tilde z_{n,v}\right\|_2^2.
\end{equation}

This MSE on normalized embeddings approximates the squared geodesic distance locally, offering more stable and bounded gradients.

\subsection{SPHERE-JEPA Objective}

Following LeJEPA, our final objective combines the invariance loss (for cross-view alignment) and SUSReg (for global hyperspherical geometry) via a balancing weight $\lambda \in [0,1]$:
\[
\mathcal L_{\mathrm{SPHERE\text{-}JEPA}}
\;=\;
(1-\lambda)\,\mathcal L_{\mathrm{inv}}
\;+\;
\lambda\,\mathcal L_{\mathrm{SUSReg}}.
\]

The choice of a spherical rather than Gaussian target raises a natural question: why not simply normalize LeJEPA's Gaussian
embeddings? 
Indeed, if $Z\sim\mathcal N(0,I_d)$, then $Z/\|Z\|\sim\operatorname{Unif}(\mathbb S^{d-1})$.
First, both methods impose their distributional constraint at the projection head, while standard downstream benchmarks evaluate the backbone without that head. Normalizing LeJEPA's head output after training, therefore, does not imply equivalent backbone features.
Second, even if the normalized head outputs were evaluated, the training objectives would differ.
An exact Gaussian vector can be written as $Z=RU$, where $U$ is uniform on the sphere and independent of the radius $R$.
With finite batches and projections, Gaussian matching also penalizes radial mismatches, although $R$ is discarded by normalization.
SUSReg directly targets the retained directions.

\section{Experiments} \label{sec:experiments}
\paragraph{Architectures.}
To evaluate across different inductive biases architectures, we use both self-attention and convolutional backbones: a primary Vision Transformer (ViT-B/14)~\citep{dinov2}, and standard ResNet-18/50 models~\citep{he2016deep}. Representations are extracted from the ViT class token and via global average pooling for the ResNets.

\paragraph{Embedding Head.}
The backbone output $g_\phi(x)$ is processed by a $3$-layer MLP projection head $h_\psi$ with GELU activations and hidden dimensions $[2048, 2048, 256]$.
In both methods, distributional regularization acts on the projected representation $h_\psi(g_\phi(x))$, whereas downstream evaluation uses the backbone representation $g_\phi(x)$ without the projection head.

\paragraph{Training Data and Views.}
We evaluate on ImageNet-1K~\citep{imagenet}, ImageNet-100, and Galaxy10~\citep{galaxy10}. For ViT-B/14 on ImageNet-1K, we use a multi-crop strategy ($V_g=2$ global, $V_l=6$ local; $V_a=8$ total views). ResNets use only two global views ($V_a=V_g=2$). We apply standard augmentations: DINO~\citep{dinov1} and SimSiam~\citep{simsiam} for ViT, and BYOL~\citep{byol} for ResNets.

\paragraph{Optimization.}
Models are trained with AdamW and cosine annealing. On ImageNet-1K, ViT-B/14 is trained for $100$ epochs (5-epoch warmup) with batch size $512$, learning rate (LR) $5 \times 10^{-4}$, and weight decay (WD) $5 \times 10^{-2}$. ResNets are trained for $200$ epochs (1-epoch warmup) with LR $5 \times 10^{-2}$ and WD $5 \times 10^{-4}$. $|\mathcal{A}|$ for SUSReg and SIGReg is fixed to $1024$ across all runs. We set the balancing weight to $\lambda=0.05$ for both methods, following the LeJEPA configuration; we did not tune this value specifically for SPHERE-JEPA.
Its sensitivity is reported in Appendix~\ref{app:lambda_sensitivity}.

\paragraph{EMA Teacher.}
When using an EMA teacher, the target prototype $\mu_n$ (Eq.~\ref{equ:prototype}) is generated by a momentum encoder and projection head. While the online network updates via gradient descent, its weights are used to update the EMA teacher with a momentum parameter $\tau = 0.996$. 
SUSReg constrains the marginal distribution of projected representations across samples, but does not control how the prediction target for an individual sample changes during training.
The EMA teacher provides more slowly changing targets by averaging encoder weights over time.
This offers a possible explanation for its optimization benefit, distinct from the geometric constraint imposed by SUSReg.

\subsection{Main Results on ImageNet Benchmarks}
\paragraph{ImageNet-1K.}
We evaluate ViT-B/14 representations after $100$ epochs on ImageNet-1K (Table~\ref{tab:in1k_merged}).
Without an EMA teacher, SPHERE-JEPA matches LeJEPA's $k$-NN accuracy and improves linear-probing accuracy by $1.8$ percentage points. 
Adding an EMA teacher yields substantial gains for both methods (Table~\ref{tab:in1k_merged}). 
This shared effect is consistent with the teacher stabilizing prediction targets; the experiment does not isolate that mechanism.
We therefore do not treat EMA as a requirement of SUSReg or of the geometric result.

\paragraph{ImageNet-100.}
We evaluate ResNet-18 backbones across four independent pretraining seeds (Table~\ref{tab:imagenet100_rn18}).
LeJEPA obtains higher mean accuracy than SPHERE-JEPA on both linear probing and $k$-NN.
Thus, the hyperspherical prior does not yield an improvement on this benchmark under the evaluated protocol.

\paragraph{Sensitivity to the regularizer weight.}
We vary the SUSReg balancing weight $\lambda$ while keeping the
architecture, training schedule, random seed, and evaluation protocol
fixed. The single-seed ablation and its results are reported in
Appendix~\ref{app:lambda_sensitivity}.

\subsection{Downstream and Specialized Evaluations}

\paragraph{Transfer Learning.} We further assess transferability via linear probing on downstream tasks using the ImageNet-1K pretrained ViT-B/14. As shown in Table~\ref{tab:transfer_vitb_all}, SPHERE-JEPA consistently improves across most datasets, achieving a higher average accuracy ($54.1\%$) than LeJEPA ($52.4\%$).

\paragraph{Texture Retrieval.} While standard benchmarks assess linear separability, our theory suggests hyperspherical representations excel with nonparametric estimators like $k$-NN. We test this directly on a controlled texture retrieval task (details in Appendix~\ref{sec:texture_details}). As reported in Table~\ref{tab:retrieval_avg}, SPHERE-JEPA significantly outperforms LeJEPA, improving mean Average Precision by $+6.36\%$ directly in the embedding space ($0.9184$ vs. $0.8548$) and $+7.34\%$ after the projection head ($0.9292$ vs. $0.8558$).

\paragraph{Galaxy10.} We evaluate ResNet-50 backbones across four independent pretraining seeds on Galaxy10 (Table~\ref{tab:galaxy10_rn50}). 
LeJEPA obtains higher mean accuracy on both linear probing and $k$-NN.
The gains observed in our other evaluations, therefore, do not extend to this setting.

\paragraph{Dense prediction.}
We pretrain ViT-Tiny backbones on ImageNet-100 with InfoNCE, VICReg, LeJEPA, and SPHERE-JEPA, then evaluate their frozen representations on ImageNet-100 and ADE20K.
Across four pretraining seeds, SPHERE-JEPA obtains the highest mean on all four metrics. 
On ADE20K, it reaches $12.88\pm0.39\%$ mIoU, compared with $11.76\pm0.24\%$ for LeJEPA.
The protocol and complete results are reported in Appendix~\ref{app:tinyvit_dense_transfer}.

\begin{table*}[t]
    \centering
    
    \caption{Frozen-backbone linear probing accuracy (\%) on downstream datasets using ViT-B/14 backbones pretrained for 100 epochs on ImageNet-1K. Avg denotes the mean across all datasets.}
    \label{tab:transfer_vitb_all}
    \resizebox{\linewidth}{!}{
    \begin{tabular}{lcccccccc}
    \toprule
    \textbf{Method} & \textbf{DTD} & \textbf{Aircraft} & \textbf{CIFAR10} & \textbf{CIFAR100} & \textbf{Flowers} & \textbf{Food} & \textbf{Pets} & \textbf{Avg} \\
    \midrule
    LeJEPA            & 60.5 & 09.0 & 82.0 & 52.1 & 45.3 & 56.9 & 60.8 & 52.4 \\
    SPHERE-JEPA       & 60.0 & 10.0 & 82.5 & 53.1 & 51.1 & 60.3 & 62.0 & \textbf{54.1} \\
    \midrule
    LeJEPA (EMA)      & 58.5 & 10.4 & 81.9 & 49.2 & 59.0 & 68.3 & 79.6 & 58.1 \\
    SPHERE-JEPA (EMA) & 60.4 & 10.5 & 82.2 & 51.8 & 52.8 & 68.6 & 81.8 & \textbf{58.3} \\
    \bottomrule
    \end{tabular}
    }
    
    \vspace{0.6cm} 
    
    \caption{Nearest-neighbor retrieval performance averaged across four texture datasets.}
    \label{tab:retrieval_avg}
    \begin{tabular}{lccccc}
    \toprule
    \textbf{Method} & \textbf{Recall@1} & \textbf{Recall@3} & \textbf{Recall@5} & \textbf{mAP} & \textbf{mAP (emb)} \\
    \midrule
    LeJEPA       & 78.2 & 91.6 & 95.1 & 85.6 & 85.5 \\
    SPHERE-JEPA  & \textbf{89.0} & \textbf{96.3} & \textbf{97.9} & \textbf{92.9} & \textbf{91.8} \\
    \bottomrule
    \end{tabular}
    
    \vspace{0.6cm} 
    
    \begin{minipage}[t]{0.48\textwidth}
        \centering
        \caption{ImageNet-1K (ViT-B/14, 100 ep.) evaluation.}
        \label{tab:in1k_merged}
        \begin{tabular}{lcc}
        \toprule
        \textbf{Method} & \textbf{$k$-NN (\%)} & \textbf{Linear (\%)} \\
        \midrule
        \multicolumn{3}{c}{\textit{Without EMA}} \\
        \midrule
        LeJEPA      & \textbf{50.4} & 59.2 \\
        SPHERE-JEPA & \textbf{50.4} & \textbf{61.0} \\
        \midrule
        \multicolumn{3}{c}{\textit{With EMA ($\tau=0.996$)}} \\
        \midrule
        LeJEPA      & 61.7 & 68.2 \\
        SPHERE-JEPA & \textbf{62.6} & \textbf{69.3} \\
        \bottomrule
        \end{tabular}
    \end{minipage}
    \hfill
\begin{minipage}[t]{0.48\textwidth}
    \centering
    \scriptsize
    \caption{ImageNet-100 (ResNet-18, 200 ep.). Mean $\pm$ sample
    standard deviation over four pretraining seeds.}
    \label{tab:imagenet100_rn18}
    \begin{tabular}{lcc}
    \toprule
    Method & $k$-NN (\%) & Linear (\%) \\
    \midrule
    LeJEPA      & $\mathbf{67.61\pm0.30}$ & $\mathbf{72.46\pm0.24}$ \\
    SPHERE-JEPA & $65.67\pm0.56$ & $71.22\pm0.58$ \\
    \bottomrule
    \end{tabular}

    \vspace{0.5cm}

    \caption{Galaxy10 (ResNet-50, 200 ep.). Mean $\pm$ sample
    standard deviation over four pretraining seeds.}
    \label{tab:galaxy10_rn50}
    \begin{tabular}{lcc}
    \toprule
    Method & $k$-NN (\%) & Linear (\%) \\
    \midrule
    LeJEPA      & $\mathbf{70.56\pm0.51}$ & $\mathbf{75.03\pm0.69}$ \\
    SPHERE-JEPA & $68.67\pm1.04$ & $72.13\pm1.39$ \\
    \bottomrule
    \end{tabular}
\end{minipage}

\end{table*}

\section{Conclusion}
\label{conclusion}

We introduced a minimax perspective on optimal SSL representations, showing that linear ridge regression, \(k\)-NN regression, and exponential-kernel ridge regression collectively characterize the uniform distribution on the sphere as an optimal representation geometry.
Motivated by this result, we introduced SPHERE-JEPA and its projection-based regularizer, SUSReg. 
Empirically, spherical uniformity improves linear separability on ImageNet1K and substantially boosts nonparametric retrieval performance. 
However, our empirical validation remains limited: our large-scale evaluation relies on a single ViT-B/14 training run without hyperparameter tuning.
Validating robustness across multiple seeds, scaling to larger batch sizes and architectures remain open challenges.

\section*{Acknowledgements}
This work was partially funded by ANRT with the CIFRE grant n° 2024/1221, Advanced Track and
Trace, and Centre Borelli.
This work was granted access to the HPC resources of IDRIS (Jean Zay supercomputer) under the allocation 2025-AD011017323 made by GENCI.

\bibliographystyle{plainnat}
\bibliography{bib}

\clearpage
\onecolumn

\setcounter{section}{0}
\renewcommand{\thesection}{\Alph{section}}

\section{Computational Resources}
\label{sec:appendix_compute}

This appendix provides the hardware specifications and computational requirements necessary to reproduce the results presented in this work.

\paragraph{Hardware Infrastructure.}
Our experiments were conducted using the following hardware setups:
\begin{itemize}
    \item \textbf{ImageNet-1K Pretraining (ViT-B/14):} All runs (with and without EMA) were executed using 4 NVIDIA A100 (80GB) GPUs.
    \item \textbf{Other Models:} Pretraining for ResNet-18 on ImageNet-100 and ResNet-50 on Galaxy10, as well as all downstream linear probing evaluations, were performed on a single NVIDIA RTX-6000 GPU.
\end{itemize}

\paragraph{Training Duration and Efficiency.}
For the 100-epoch pretraining of a ViT-B/14 on ImageNet-1K, the execution time was approximately 2.5 days (60 hours) per run. It is important to note that \textbf{SPHERE-JEPA} does not incur any additional computational overhead compared to the LeJEPA baseline.

\paragraph{Software Environment.}
All models were implemented using PyTorch and trained using the AdamW optimizer. We will release the core implementation and training scripts to ensure full reproducibility.

\section{Existing Assets and Licenses}
\label{app:licenses}
The datasets and software libraries used in this work are publicly available and strictly used for academic research purposes, in compliance with their respective licenses. PyTorch is released under the BSD-style license. ImageNet-1K and ImageNet-100 are subject to the ImageNet terms of access for non-commercial research. The Galaxy10 DECals dataset is derived from the DESI Legacy Imaging Surveys and Galaxy Zoo, distributed under the MIT license and CC BY 4.0, respectively.

\section{Sensitivity to the Regularizer Weight}
\label{app:lambda_sensitivity}

We vary the balancing weight $\lambda$ in
\[
\mathcal L_{\mathrm{SPHERE\text{-}JEPA}}
=(1-\lambda)\mathcal L_{\mathrm{inv}}
+\lambda\mathcal L_{\mathrm{SUSReg}}.
\]
All runs, reported in Table~\ref{tab:lambda_sensitivity}, use the same ResNet-18 architecture, ImageNet-100 training
schedule, random seed, and evaluation protocol. 
Thus, this experiment isolates the effect of $\lambda$ for one pretraining seed; it is distinct from the four-seed comparison in Table~\ref{tab:imagenet100_rn18}.

\begin{table}[H]
\centering
\caption{ImageNet-100 top-1 accuracy (\%) for different values of the
SUSReg balancing weight $\lambda$ (one pretraining seed).}
\label{tab:lambda_sensitivity}
\begin{tabular}{ccc}
\toprule
$\lambda$ & Linear probe & $k$-NN \\
\midrule
$0.003$ & $68.8$ & $61.7$ \\
$0.005$ & $69.6$ & $64.2$ \\
$0.007$ & $70.6$ & $64.4$ \\
$0.009$ & $70.7$ & $65.7$ \\
$0.01$  & $70.5$ & $65.2$ \\
$0.03$  & $70.8$ & $65.7$ \\
$0.05$  & $70.8$ & $64.8$ \\
$0.07$  & $70.0$ & $64.3$ \\
$0.09$  & $68.8$ & $62.4$ \\
\bottomrule
\end{tabular}
\end{table}

Performance varies with $\lambda$: both metrics decrease at the smallest and largest tested values.
Between $0.007$ and $0.05$, linear-probe accuracy ranges from $70.5\%$ to $70.8\%$, while $k$-NN accuracy ranges from $64.4\%$ to $65.7\%$. 
The best linear-probe accuracy is obtained at $0.03$ and $0.05$; the best $k$-NN accuracy is obtained at $0.009$ and $0.03$.

\section{ViT-Tiny Evaluation and Dense Prediction}
\label{app:tinyvit_dense_transfer}

We pretrain ViT-Tiny backbones on ImageNet-100 using InfoNCE, VICReg, LeJEPA, and SPHERE-JEPA.
For each method, we evaluate the pretrained representations with a linear probe and $k$-NN classification on ImageNet-100.
To assess spatial transfer, we freeze the backbone and train a linear segmentation head on its patch features for ADE20K. 
Table~\ref{tab:tinyvit_architecture_dense_transfer} reports the mean and sample standard deviation over four independent pretraining seeds.

\begin{table}[H]
\centering
\small
\setlength{\tabcolsep}{4pt}
\caption{
Evaluation of ViT-Tiny representations pretrained on ImageNet-100.
We report top-1 linear-probe and $k$-NN accuracy on ImageNet-100, and linear segmentation-probe performance with a frozen backbone on ADE20K.
All values are percentages (mean $\pm$ sample standard deviation over four independent pretraining seeds). 
Bold indicates the highest mean in each column.
}
\label{tab:tinyvit_architecture_dense_transfer}
\begin{tabular}{lcccc}
\toprule
& \multicolumn{2}{c}{ImageNet-100}
& \multicolumn{2}{c}{ADE20K} \\
\cmidrule(lr){2-3}\cmidrule(lr){4-5}
Method & Linear probe & $k$-NN & mIoU & Pixel acc. \\
\midrule
InfoNCE
& $51.93 \pm 4.01$
& $56.93 \pm 0.58$
& $7.81 \pm 0.45$
& $55.26 \pm 0.43$ \\
VICReg
& $57.74 \pm 0.51$
& $56.01 \pm 0.48$
& $9.20 \pm 0.56$
& $56.76 \pm 0.59$ \\
LeJEPA
& $70.71 \pm 0.38$
& $67.39 \pm 0.61$
& $11.76 \pm 0.24$
& $58.94 \pm 0.26$ \\
\midrule
SPHERE-JEPA (SUSReg)
& $\mathbf{71.91} \pm 0.90$
& $\mathbf{69.03} \pm 0.43$
& $\mathbf{12.88} \pm 0.39$
& $\mathbf{59.92} \pm 0.32$ \\
\bottomrule
\end{tabular}
\end{table}

SPHERE-JEPA shows the highest observed mean on all four metrics.
Relative to LeJEPA, its gains are 1.20 percentage points in linear-probe accuracy, 1.64 points in $k$-NN accuracy, 1.12 points in ADE20K mIoU, and 0.98 points in pixel accuracy.

\section{Texture Retrieval Details}
\label{sec:texture_details}

\subsection{Problem Setup}

We consider a nonparametric retrieval task designed to evaluate the geometry of learned representations. Given a query image, the objective is to retrieve another view of the same texture instance among visually similar samples.

In this setting, each data point is defined as a set of transformed views generated from the same underlying texture. This formulation encourages representations of transformed views to remain close in the embedding space while preserving separation across different textures.

The resulting evaluation protocol depends on the relative distanceamong nearby samples in the representation space, making it particularly well-suited toor analyzing how representation regularization affects nearest-neighbor retrieval.

\subsection{Datasets}

We instantiate the retrieval protocol using four procedural texture datasets: \textit{Disk}, \textit{Cloud}, \textit{Flake}, and \textit{Wood} (see Figure~\ref{fig:textures}). Each dataset is generated from a distinct stochastic process, defining a family of textures with shared statistical properties.

Disk and Flake textures are derived from heavy-tailed noise with additional smoothing or structured filtering, respectively. Cloud textures are generated from blurred Brownian motion, yielding smooth and self-similar patterns. Wood textures are obtained from Perlin noise using the \texttt{noise} library\footnote{\url{https://pypi.org/project/noise/}}.

Within each dataset, samples are generated from the same underlying stochastic process with different random seeds, yielding images that share similar statistical structure yet remain visually distinct.

For each texture family, we generate $10{,}000$ samples for training, $500$ for validation, and $10{,}000$ for testing, using independent random seeds to ensure no overlap between splits.
\begin{figure}
  \centering

  \includegraphics[width=0.32\linewidth]{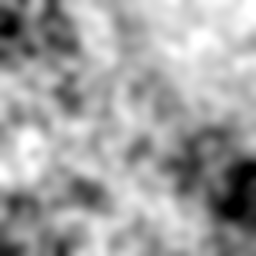}
  \includegraphics[width=0.32\linewidth]{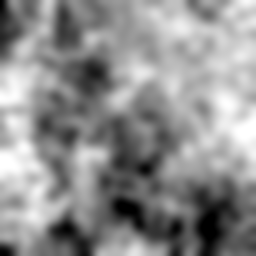}
  \includegraphics[width=0.32\linewidth]{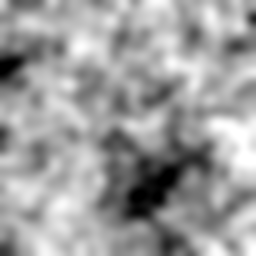}

  \medskip

  \includegraphics[width=0.32\linewidth]{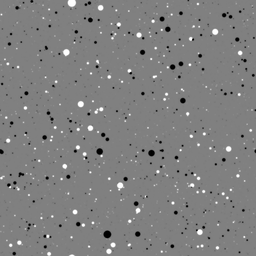}
  \includegraphics[width=0.32\linewidth]{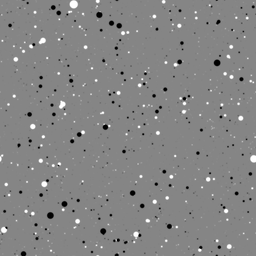}
  \includegraphics[width=0.32\linewidth]{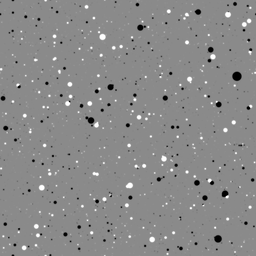}

  \medskip

  \includegraphics[width=0.32\linewidth]{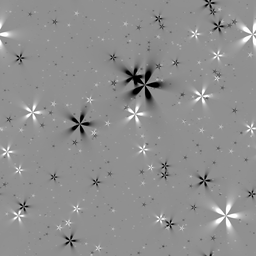}
  \includegraphics[width=0.32\linewidth]{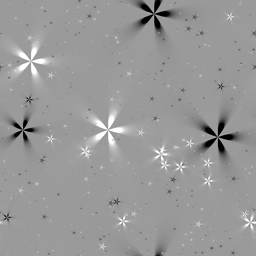}
  \includegraphics[width=0.32\linewidth]{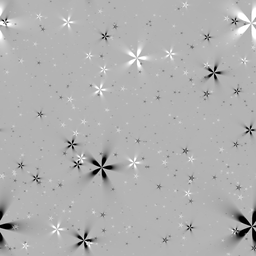}

  \medskip

  \includegraphics[width=0.32\linewidth]{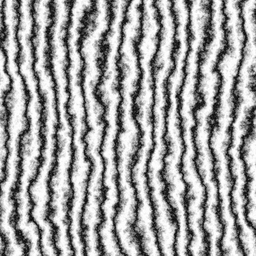}
  \includegraphics[width=0.32\linewidth]{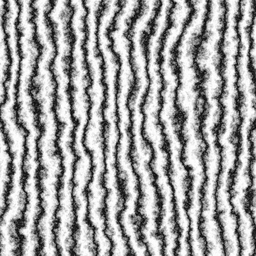}
  \includegraphics[width=0.32\linewidth]{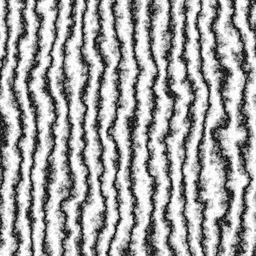}

  \caption{Examples of procedural textures used for retrieval evaluation. Each row corresponds to a texture family (top to bottom: cloud, disk, flake, wood). Images within a row are generated from the same stochastic process with different random seeds, resulting in strong statistical similarity but distinct visual realizations.}
  \label{fig:textures}
\end{figure}

\subsection{Data Augmentation and View Generation}
\label{subsec:crop_texture}

To generate multiple views from each texture image, we apply stochastic spatial and photometric augmentations. These transformations introduce significant variability while preserving the underlying texture statistics across views.

Each view is generated using a random affine transformation implemented with Kornia. The transformation parameters include a rotation angle uniformly sampled in $[0,360]$ degrees, a translation of up to $20\%$ of the image size along each spatial dimension, an isotropic scaling factor, and a random shear transformation.

The transformed image is obtained by zero-padding outside the image boundaries and is resized to a fixed resolution of $224\times224$ pixels. Figure~\ref{fig:affine_and_crops} illustrates both the effective source regions induced by independently sampled transformations and the corresponding augmented views.

Each view is further modified through photometric augmentations, including brightness and contrast adjustments, additive noise, and random erasing. All images are converted to grayscale, clamped to valid intensity ranges, and normalized before being fed to the network.

Despite strong spatial and photometric variations, the resulting views preserve consistent texture statistics while exhibiting substantial local variability, as illustrated in Figure~\ref{fig:affine_and_crops}.

Unless otherwise specified, we use $V_g = 2$ views per image. The same augmentation pipeline is used at test time to ensure consistency between training and evaluation.

\begin{figure}[t]
\centering
\begin{minipage}[c]{0.48\linewidth}
    \centering
    \includegraphics[width=\linewidth]{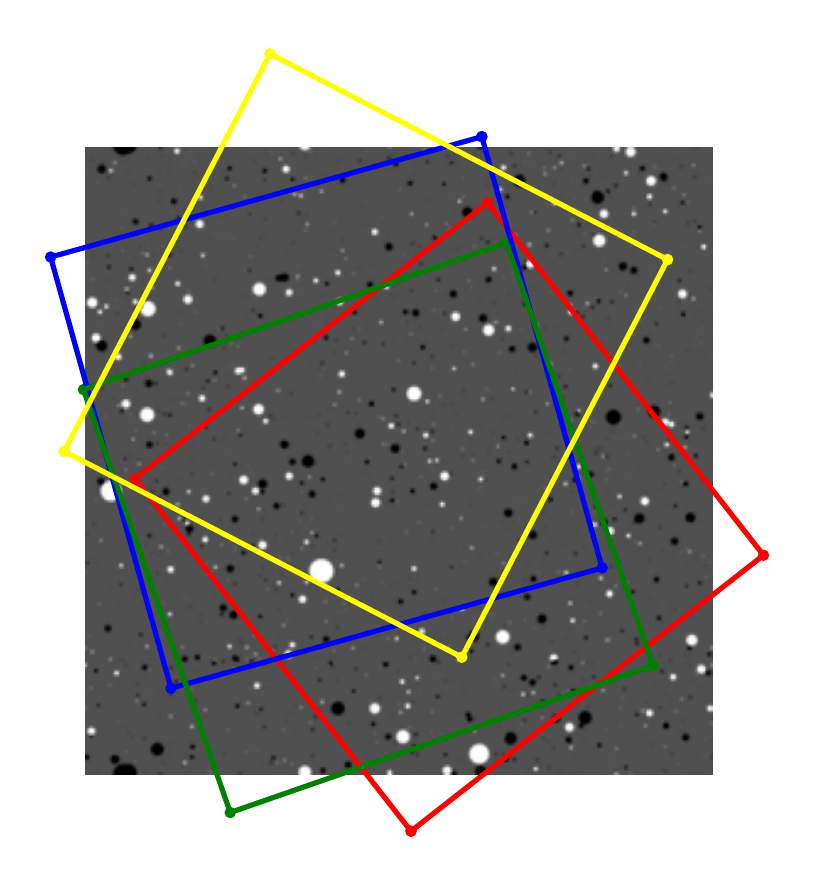}
\end{minipage}
\hfill
\begin{minipage}[c]{0.48\linewidth}
    \centering
    \includegraphics[width=0.49\linewidth]{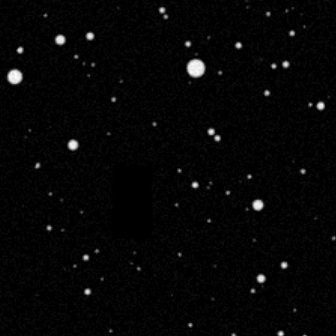}
    \includegraphics[width=0.49\linewidth]{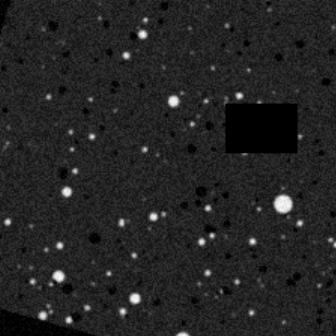}

    \vspace{2mm}

    \includegraphics[width=0.49\linewidth]{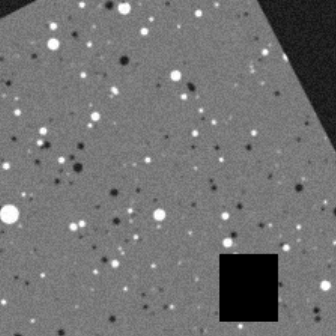}
    \includegraphics[width=0.49\linewidth]{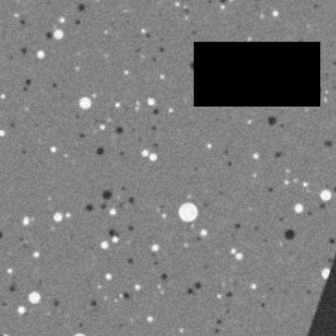}
\end{minipage}

\caption{\textbf{Left:} effective source regions induced by independently sampled random affine transformations.  
\textbf{Right:} corresponding augmented views obtained after affine and photometric transformations.}
\label{fig:affine_and_crops}
\end{figure}

\subsection{Retrieval Results}

We evaluate retrieval performance using a nearest-neighbor protocol. For each query image, the goal is to retrieve another view of the same texture instance among a set of candidate samples.

Evaluation is performed within mini-batches of size $B=100$. For each query, similarity is computed against the other samples in the batch, and retrieval performance is measured based on the ranking of these candidates.

We report nearest-neighbor retrieval performance using Recall@K and mean Average Precision (mAP), evaluated both on the projection head outputs and directly on backbone embeddings.

Tables~\ref{tab:retrieval_disk},~\ref{tab:retrieval_flake},~\ref{tab:retrieval_cloud},~\ref{tab:retrieval_wood} report results for each dataset, while Table~\ref{tab:retrieval_avg} summarizes the average performance across all four datasets.

These results highlight the importance of representation geometry in nonparametric settings, with hyperspherical regularization leading to improved performance when evaluated directly in the embedding space.

\begin{table}[h]
\centering
\caption{Retrieval performance on the Disk texture dataset.}
\label{tab:retrieval_disk}
\begin{tabular}{lccccc}
\toprule
Method & Recall@1 & Recall@3 & Recall@5 & mAP & mAP (emb) \\
\midrule
byol       & 95.7 & 98.7 & 99.1 & \textbf{97.3} & 61.2 \\
spherejepa & 89.1 & 95.9 & 97.5 & 92.8 & \textbf{90.6} \\
lejepa     & 82.3 & 92.4 & 95.2 & 87.9 & 88.0 \\
vicreg     & 67.0 & 82.1 & 86.9 & 75.9 & 75.7 \\
\bottomrule
\end{tabular}
\end{table}

\begin{table}[h]
\centering
\caption{Retrieval performance on the Flake texture dataset.}
\label{tab:retrieval_flake}
\begin{tabular}{lccccc}
\toprule
Method & Recall@1 & Recall@3 & Recall@5 & mAP & mAP (emb) \\
\midrule
byol & 88.4 & 96.0 & 97.9 & \textbf{92.5} & \textbf{89.3} \\
spherejepa & 83.0 & 94.6 & 97.2 & 89.2 & 89.0 \\
lejepa & 79.5 & 93.2 & 96.3 & 86.9 & 87.5 \\
vicreg & 69.7 & 86.6 & 92.0 & 79.3 & 80.3 \\
\bottomrule
\end{tabular}
\end{table}

\begin{table}[h]
\caption{Retrieval performance on the Cloud texture dataset.}
\label{tab:retrieval_cloud}\centering
\begin{tabular}{lccccc}
\toprule
Method & Recall@1 & Recall@3 & Recall@5 & mAP & mAP (emb) \\
\midrule
byol  & 93.6 & 97.3 & 98.2 & \textbf{95.6} & 80.4 \\
spherejepa  & 90.1 & 95.7 & 97.3 & 93.3 & \textbf{91.8} \\
lejepa & 87.4 & 94.8 & 96.7 & 91.5 & 90.6 \\
vicreg  & 81.8 & 90.5 & 93.3 & 86.9 & 87.4 \\
\bottomrule
\end{tabular}
\end{table}

\begin{table}[h]
\centering
\caption{Retrieval performance on the Wood texture dataset.}
\label{tab:retrieval_wood}
\begin{tabular}{lccccc}
\toprule
Method & Recall@1 & Recall@3 & Recall@5 & mAP & mAP (emb) \\
\midrule
byol & 98.9 & 99.8 & 99.9 & \textbf{99.4} & 93.0 \\
spherejepa & 93.8 & 99.0 & 99.7 & 96.4 & \textbf{95.9} \\
lejepa & 63.8 & 86.0 & 92.3 & 76.0 & 75.7 \\
vicreg & 16.3 & 39.7 & 56.2 & 34.6 & 35.0 \\
\bottomrule
\end{tabular}
\end{table}

\section{Worst-case integrated squared bias for kernel ridge regression}
\label{app:worstcasekrr}
We briefly recall the population and empirical kernel ridge regression
(KRR) formulations.

Let
\[
K(x,y)=\exp(\kappa x^\top y),
\qquad \kappa>0,
\]
let \(\mathcal H_K\) denote the associated RKHS
(see, e.g., \cite{scholkopf2002learning}),
and let \(p\) be a probability distribution on
\(\mathcal M\subseteq\mathbb R^d\).

For \(g_\star\in\mathcal H_K\), the empirical KRR estimator based on
\[
x_1,\dots,x_B\overset{i.i.d.}{\sim}p
\]
is defined by
\[
\widehat g_{B,\lambda}
=
\arg\min_{\hat g\in\mathcal H_K}
\frac1B\sum_{i=1}^B
(\hat g(x_i)-g_\star(x_i))^2
+
\lambda\|\hat g\|_{\mathcal H_K}^2,
\qquad \lambda>0.
\]

Equivalently, in operator form (see, e.g., \cite{scholkopf2002learning}),
\[
\widehat g_{B,\lambda}
=
(T_B+\lambda I)^{-1}T_B g_\star,
\]
where
\[
T_B g
=
\frac1B
\sum_{i=1}^B
K(x_i,\cdot)\,g(x_i)
\]
denotes the empirical covariance operator.

The corresponding population covariance operator is
\[
T_p g
=
\int_{\mathcal M}
K(x,\cdot)\,g(x)\,
p(x)\,\mathrm{dvol}(x).
\]

Assume that
\[
\int_{\mathcal M}
K(x,x)\,
p(x)\,\mathrm{dvol}(x)
<\infty.
\]
Under this condition, the empirical covariance operator
converges almost surely to the population covariance operator
as \(B\to\infty\)~\citep{de2005learning}.
Consequently, in the large-sample regime,
the integrated squared bias of empirical KRR is asymptotically governed
by its population counterpart.
Therefore, throughout the remainder of the analysis,
we study the population bias functional directly.

The population integrated squared bias associated with \(g_\star\) is
\[
\mathrm{ISB}(g_\star, p)
=
\int_{\mathcal M}
\left[
\lambda(T_p+\lambda I)^{-1}g_\star(x)
\right]^2
p(x)\,\mathrm{dvol}(x).
\]

We consider the source class
\[
\mathcal G_p
=
\left\{
g_\star=T_p f
\;:\;
\|f\|_{L^2(p)}\le1
\right\}.
\]

We define the worst-case population integrated squared bias by
\[
\mathcal E(p)
=
\sup_{g_\star\in\mathcal G_p}
\mathrm{ISB}_p(g_\star).
\]

Since \(K\) is positive definite,
\[
|K(x,y)|^2\le K(x,x)K(y,y).
\]
Hence the assumption
\[
\int_{\mathcal M}
K(x,x)\,
p(x)\,\mathrm{dvol}(x)
<\infty
\]
implies
\[
\int_{\mathcal M}\int_{\mathcal M}
|K(x,y)|^2
p(x)p(y)\,
\mathrm{dvol}(x)\mathrm{dvol}(y)
<\infty.
\]
Therefore, the population covariance operator \(T_p\) is
Hilbert--Schmidt, hence compact, on \(L^2(p)\).
Moreover, since \(K\) is symmetric and positive definite,
\(T_p\) is self-adjoint and positive.

Consequently, there exist eigenvalues
\[
\mu_1\ge\mu_2\ge\cdots\ge0,
\qquad
\mu_j\to0,
\]
and an orthonormal basis
\[
(e_j)_{j\ge1}
\]
of \(L^2(p)\) such that
\[
T_p e_j=\mu_j e_j,
\qquad j\ge1.
\]

\begin{restatable}[Spectral form of the population worst-case ISB]{lem}{spectralformpopworst}
\label{lemma:spectralformpopworst}
Let \((\mu_j,e_j)_{j\ge1}\) be the eigendecomposition of \(T_p\). Then
\[
\mathcal E(p)
=
\max_{j\ge1}
\left(
\frac{\mu_j\lambda}{\mu_j+\lambda}
\right)^2.
\]  
\end{restatable}

\begin{restatable}[Uniform spherical distributions minimize the population worst-case KRR bias]{prop}{propkrr}
Assume that \(p\) is a density on \(\mathcal M\) satisfying 
\[ \int_{\mathcal M} \|x\|^2 p(x)\,\mathrm{dvol}(x)=1.\]
Then minimizing the worst-case population bias \(\mathcal E(p)\) is equivalent to minimizing the top eigenvalue \(\mu_1\) of the population covariance operator \(T_p\).

Moreover, for the exponential kernel
\[
K(x,y)=e^{\kappa x^\top y},
\]
the minimizer is
\[ p^\star = \mathrm{Unif}(\mathbb S^{d-1}).
\]
\end{restatable}

\section{Proofs: Kernel Ridge Regression}

\spectralformpopworst*
\begin{proof}

Expand
\[
f
=
\sum_{j\ge1} a_j e_j,
\qquad
\sum_{j\ge1}a_j^2\le1.
\]

Since
\(
g_\star=T_p f,
\)
we have
\[
g_\star
=
\sum_{j\ge1}
\mu_j a_j e_j.
\]

Moreover,
\[
\lambda(T_p+\lambda I)^{-1}g_\star
=
\sum_{j\ge1}
\frac{\lambda\mu_j}{\mu_j+\lambda}
a_j e_j.
\]

Therefore,
\[
\left\|
\lambda(T_p+\lambda I)^{-1}g_\star
\right\|_{L^2(p)}^2
=
\sum_{j\ge1}
\left(
\frac{\lambda\mu_j}{\mu_j+\lambda}
\right)^2
a_j^2.
\]

Taking the supremum over
\(
\sum_{j\ge1}a_j^2\le1,
\)
the optimum is attained by concentrating all mass on the eigendirection maximizing
\(
\left(
\frac{\lambda\mu_j}{\mu_j+\lambda}
\right)^2.
\)

Hence
\[
\mathcal E(p)
=
\max_{j\ge1}
\left(
\frac{\lambda\mu_j}{\mu_j+\lambda}
\right)^2.
\]

\end{proof}

\propkrr*
\begin{proof}

Define
\[
h_\lambda(\mu)
=
\left(
\frac{\mu\lambda}{\mu+\lambda}
\right)^2.
\]

Since \(h_\lambda\) increases strictly on \([0,+\infty)\), Lemma~\ref{lemma:spectralformpopworst} implies
\[
\mathcal E(p)
=
h_\lambda(\mu_1).
\]
Therefore minimizing \(\mathcal E(p)\) is equivalent to minimizing $\mu_1$.

Using the expansion
\[
e^{\kappa x^\top y}
=
\sum_{m=0}^\infty
\frac{ \kappa \langle x,y\rangle^m}{m!},
\]
the kernel admits a feature representation
\[
K(x,y)
=
\langle\Phi(x),\Phi(y)\rangle.
\]

By the spectral characterization of compact positive self-adjoint operators
(see, e.g., \cite{conway2019course}),
\[
\mu_1
=
\sup_{\|a\|=1}
\langle T_p a,a\rangle.
\]

Using the feature representation
\[
K(x,y)=\langle\Phi(x),\Phi(y)\rangle,
\]
Rewriting the integral that defines $T_p$ as an expectation and thanks to the feature representation, the covariance operator can be written as: 
\[
T_p a
=
\mathbb E_{X\sim p}
\big[
\langle a,\Phi(X)\rangle\,\Phi(X)
\big].
\]
Therefore,
\[
\langle T_p a,a\rangle
=
\mathbb E_{X\sim p}
\big[
\langle a,\Phi(X)\rangle^2
\big].
\]
and hence
\begin{equation} \label{equ:mu1}
\mu_1
=
\sup_{\|a\|=1}
\mathbb E_{X\sim p}
\big[
\langle a,\Phi(X)\rangle^2
\big].
\end{equation}

Thus \(\mu_1\) corresponds to the largest variance direction in feature space.

Next, the kernel is rotationally invariant:
\[
K(Qx,Qy)=K(x,y),
\qquad
Q\in O(d).
\]

Let \(dQ\) denote the normalized Haar probability measure on the orthogonal
group \(O(d)\).
Given any admissible distribution \(p\), define its rotational symmetrization by
\[
\bar p(A)
=
\int_{O(d)}
Q^* p(A)\,
dQ,
\]
with $Q^*$ the pullback of $p$ by $Q$, for every measurable set \(A\subseteq\mathcal M\).

By construction, \(\bar p\) is rotationally invariant and remains a
probability distribution.
Since orthogonal transformations preserve Euclidean norms,
\[
\int_{\mathcal M}
\|x\|^2
\bar p(x)\,\mathrm{dvol}(x)
=
1.
\]

Moreover, by \eqref{equ:mu1} and the fact that the supremum of linear functionals is convex,
\(p\mapsto \mu_1(T_p)\) is convex. Therefore,
\[
\mu_1(T_{\bar p})
\le
\int_{O(d)}
\mu_1(T_{Q^*p})
\,dQ,
\]

Since the kernel is rotationally invariant,
\[
\mu_1(T_{Q^*p})
=
\mu_1(T_p),
\]
and hence
\[
\mu_1(T_{\bar p})
\le
\mu_1(T_p).
\]

Therefore, without loss of generality, we may restrict attention to
rotationally invariant distributions \(p\).

Let \(X\sim p\). Since \(p\) is rotationally invariant, a standard
spherical decomposition (see, e.g., \cite{vershynin2018high}) yields
\(
X=RU,
\)
where
\(
U\sim \mathrm{Unif}(\mathbb S^{d-1}),
\) with
\(R\ge0,
\)
and \(R\) independent of \(U\), and satisfying
\(
\mathbb E[R^2]=1.
\)

Let
\(
Y=SV
\)
be an independent copy of \(X\), where
\(
V\sim \mathrm{Unif}(\mathbb S^{d-1})
\)
is independent of \(S\).

Then
\[
X^\top Y
=
RS\,U^\top V,
\]
and
\[
\mathbb E[e^{\kappa X^\top Y}]
=
\mathbb E[g_d(RS)],
\]
where
\[
g_d(t)
=
\mathbb E_{U,V}
[e^{t \kappa U^\top V}].
\]

Since
\[
e^{t \kappa U^\top V}
=
\sum_{m=0}^\infty
\frac{t^m \kappa^m(U^\top V)^m}{m!},
\]
symmetry implies
\[
g_d(t)
=
\sum_{m=0}^\infty
c_m t^{2m},
\qquad
c_m\ge0.
\]

Therefore,
\[
\mathbb E[g_d(RS)]
=
\sum_{m=0}^\infty
c_m
(\mathbb E[R^{2m}])^2.
\]

By Jensen's inequality,
\[
\mathbb E[R^{2m}]
\ge
(\mathbb E[R^2])^m
=
1.
\]

Hence
\[
\mathbb E[g_d(RS)]
\ge
\sum_{m=0}^\infty c_m
=
g_d(1).
\]

Equality holds only if
\[
\mathbb E[R^{2m}]=1,
\qquad m\ge1.
\]
By Lemma~\ref{lemma:Rconstant}, this implies
\[
R=1
\qquad\text{a.s.}
\]

Since we already reduced the problem to rotationally invariant
distributions, any admissible optimizer must satisfy
\[
X=RU,
\qquad
U\sim\mathrm{Unif}(\mathbb S^{d-1}).
\]

Hence the condition \(R=1\) almost surely implies
\[
X\sim\mathrm{Unif}(\mathbb S^{d-1}).
\]

Therefore,
\[
p^\star
=
\mathrm{Unif}(\mathbb S^{d-1})
\]
minimizes \(\mu_1\), and consequently minimizes $\mathcal E(p)$.
\end{proof}

\begin{lemma} \label{lemma:Rconstant}
Let \(R\ge0\) be a random variable such that
\[
\mathbb E[R^{2m}]=1,
\qquad m\ge1.
\]
Then
\[
R=1
\qquad\text{almost surely}.
\]
\end{lemma}

\begin{proof}
Taking \(m=1\) and \(m=2\) gives
\[
\mathbb E[R^2]=1,
\qquad
\mathbb E[R^4]=1.
\]
Hence
\[
\operatorname{Var}(R^2)
=
\mathbb E[R^4]-\mathbb E[R^2]^2
=
0,
\]
so \(R^2=1\) almost surely. Since \(R\ge0\), we conclude that
\(R=1\) almost surely.
\end{proof}

\begin{remark}
The normalization
\[
\int_{\mathcal M}\|x\|^2p(x)\,\mathrm{dvol}(x)=1
\]
rules out the collapsed solution \(p=\delta_0\), for which
\(K(x,y)\equiv\phi(0)\). The argument extends to dot-product kernels
\(K(x,y)=\phi(x^\top y)\) with
\[
\phi(t)=\sum_{\ell\ge0}a_\ell t^\ell,
\qquad a_\ell\ge0.
\]
It is sufficient that \(a_{2j}>0\) for some \(j\ge2\). Indeed, equality
then requires \(\mathbb E[R^{2j}]=(\mathbb E[R^2])^j=1\), and strict
convexity implies \(R=1\) almost surely.
\end{remark}

\section{Uniform Distributions Minimize $k$-NN Regression Bias Modified Fisher Information Term}
\label{proof:knn}
The work of Balestriero and LeCun \cite{balestriero2025lejepaprovablescalableselfsupervised} 
 makes a theoretical case for isotropic Gaussian embeddings by demonstrating their optimal performance in both linear probing (ordinary least-squares regression) and nonlinear probing ($k$-NN and Nadaraya-Watson kernel regression). Specifically, they show that isotropic embeddings (not necessarily Gaussian) minimize the Tikhonov-regularized loss of ridge regression models fit to an arbitrary target on the embeddings using the linear kernel. Moreover, they show that isotropic Gaussian embeddings minimize, for a fixed number of training examples, the integrated square bias term of the bias-variance decomposition of the mean-squared error loss of $k$-NN regression and Nadaraya-Watson kernel regression under certain isotropy constraints (the prior introduced on p. 29 of that work).  However, their proofs involve assumptions that rule out common choices of latent representations, such as spherical embeddings. For example, they linearize the integrated square bias term with a Euclidean Taylor expansion that imposes smoothness constraints on the embedding distribution and regression target function over $\mathbb{R}^d$ that fail for functions defined on the sphere (ibid., Lemmas 4 and 5). Moreover, their optimality results (ibid., Theorems 1 and 7-9) depend on the fact that the leading term of this expansion, the Euclidean Fisher information integral
\begin{equation}
\label{eq:fisher}
\int_{\mathbb{R}^d} ||\nabla \log p(x)||_2^2\, p(x)\, \mathrm{d}x,
\end{equation}
is minimized (for fixed covariance matrix) by the Gaussian distribution (ibid., Lemma 6)–and, with additional scalar constraints on the distribution (ibid., Theorem 9), the isotropic Gaussian. But this integral is not defined when $p$ is, for instance, a uniform probability measure on the sphere $\mathbb{S}^{d-1}$ in $\mathbb{R}^d$. The uniform measure on the sphere is singular with respect to the Lebesgue measure on the ambient Euclidean space $\mathbb{R}^d$ and does not have a probability density function with respect to the Lebesgue measure on $\mathbb{R}^d$. Moreover, the leading term of the bias of $k$-NN regression is not proportional to \eqref{eq:fisher} when one accounts for the spatially varying $k$-NN radius.

The integrated squared bias of $k$-NN regression and kernel density estimation models on more general Riemannian manifolds can be performed using an expansion in covariant Taylor series, which does not impose these smoothness requirements in the ambient space. Indeed, this approach was used to estimate kernel density estimation and $k$-NN density estimation–where the regression target is the embedding distribution itself–in \citep{pelletier2005kernel} and \citep{henry2011knearestneighbordensityestimation}, respectively.

This more general setting tells a different story.

Let $\mathcal{M} \subseteq \mathbb{R}^d$ be an $m$-dimensional Riemannian manifold with positive injectivity radius $0<r_{\min} \leq \text{inj}_\mathcal{M}(x)$ at all points $x\in\mathcal{M}$ and volume form $\mathrm{dvol}(x)$. Let $p$ be a $\mathcal{C}^3$-smooth probability density on $\mathcal{M}$. For each point $x\in\mathcal{M}$, let $B_r(x) \subseteq \mathcal{M}$ be the geodesic ball of radius $r$. 

Given a target function $f\in C^3(\mathcal M)$, define the
population-radius local average
\[
\widehat f_r(x)
=
\mathbb E_{Y\sim p}\!\left[f(Y)\mid Y\in B_r(x)\right]
=
\frac{\int_{B_r(x)}f(y)p(y)\,\mathrm{dvol}(y)}
     {\int_{B_r(x)}p(y)\,\mathrm{dvol}(y)}.
\]
We evaluate it at the deterministic radius $r=r_k(x)$ defined by
\[
\int_{B_{r_k(x)}(x)}p(y)\,\mathrm{dvol}(y)=\frac{k}{n}.
\]
Throughout, $k\to\infty$ and $k/n\to0$.

The pointwise bias of our estimate is simply
$$\mathrm{bias}(x; r_k(x)) = \mathbb{E}_{y\sim p}\bigl[f(y) \mid y \in B_r(x)\bigr] - f(x) = \widehat{f}_r(x) - f(x);$$ the integrated squared bias, 
\begin{equation}
\label{eq:isb}
\text{ISB}(f,p) = \int_\mathcal{M} \mathrm{bias}(x; r_k(x))^2 \, p(x)\,\mathrm{dvol}(x).
\end{equation}

\begin{restatable}[Riemannian analog of \cite{balestriero2025lejepaprovablescalableselfsupervised}, Lemma 4]{lem}{lemmafour}
\label{lemma4}
For $f,p\in C^3(\mathcal M)$ with $p(x)>0$, the pointwise bias satisfies
\[
\mathrm{bias}(x;r_k(x))
=
\frac{r_k(x)^2}{2(m+2)}
\left(
\Delta_{\mathcal M}f(x)
+2\langle\nabla f(x),\nabla\log p(x)\rangle_g
\right)
+O(r_k(x)^3).
\]
\end{restatable}

We shall refer to the term $\Delta_{\mathcal{M}}f(x)$ as the wiggliness bias and $\langle \nabla f(x), \nabla p(x)\rangle_g/p(x)$ as the design bias. 

\begin{restatable}[Riemannian analog of \cite{balestriero2025lejepaprovablescalableselfsupervised}, Theorem 7]{prop}{theoremseven}
\label{theorem7}
Let $\mathcal M$ be a compact, smooth, $m$-dimensional Riemannian manifold without boundary. 
Let $f,p\in C^3(\mathcal M)$ be fixed, with $p>0$ and $\int_{\mathcal M}p\,\mathrm{dvol}=1$.
Assume $k\to\infty$ and $k/n\to0$. Here $r_k(x)$ is the deterministic population radius satisfying
$P(B_{r_k(x)}(x))=k/n$.
The leading term of the integrated squared bias of $k$-NN regression \eqref{eq:isb} is
\begin{equation}
\label{isb_expanded}
    \mathrm{ISB}(f,p;k,n) = A(m,k,n) \int_{\mathcal{M}} \left( \Delta_{\mathcal{M}} f(x) + 2\frac{\langle \nabla f(x), \nabla p(x) \rangle_g}{p(x)} \right)^2 p(x)^{1 - 4/m} \, \mathrm{dvol}(x)+
O\left(\left(\frac{k}{n}\right)^{5/m}\right),
\end{equation}
where $A(m,k,n) \coloneq \frac{1}{4(m+2)^2} \left( \frac{k \Gamma(1 + m/2)}{n \pi^{m/2}} \right)^{4/m}.$
\end{restatable}
\begin{remark}
The term $p(x)^{-4/m}$ does not appear in the result of \cite{balestriero2025lejepaprovablescalableselfsupervised}; it arises due to the spatially varying bandwidth of a $k$-NN regression.
\end{remark}
\begin{remark}
Under the Laplacian constraint used below, the leading worst-case functional is finite for each fixed positive density on a compact manifold.
Its asymptotic value is strictly larger for nonuniform densities than for the uniform density.
\end{remark}

Using a minimax approach \citep{micchelli1979design, bach2017equivalence}, we find that the optimal design for minimizing $k$-NN regression is a uniform distribution.

\begin{restatable}[Asymptotic minimax optimality of the uniform design for
\(k\)-NN]{prop}{propknn}
\label{prop:knn_uniform_minimax}
Let \(\mathcal M\subset\mathbb R^d\) be a compact, connected, \(m\)-dimensional Riemannian submanifold without boundary. 
Let \(c>0\), \(k=k(n)\to\infty\), and \(k/n\to0\).
Suppose that \(M_n\to\infty\) and
\[
M_n\left(\frac{k}{n}\right)^{1/m}\to0.
\]
Define
\[
|f|_{C^3}
=\max_{1\le j\le3}\|\nabla^j f\|_\infty,
\qquad
\mathcal G_n
=\left\{
f\in C^3(\mathcal M):
\|\Delta_{\mathcal M}f\|_2\le c,\;
|f|_{C^3}\le M_n
\right\}.
\]
Set \(q=1-4/m\), \(p_0=\operatorname{vol}(\mathcal M)^{-1}\), and
\[
A(m,k,n)
=\frac{1}{4(m+2)^2}
\left(
\frac{k\,\Gamma(1+m/2)}{n\pi^{m/2}}
\right)^{4/m}.
\]
For every fixed positive \(p\in C^3(\mathcal M)\) with
\(\int_{\mathcal M}p\,\mathrm{dvol}=1\),
\[
\liminf_{n\to\infty}
\frac{\sup_{f\in\mathcal G_n}\mathrm{ISB}(f,p;k,n)}
     {A(m,k,n)}
\ge c^2\|p^q\|_\infty.
\]
For \(p=p_0\), the normalized supremum converges to \(c^2p_0^q\).
Moreover, every nonconstant \(p\) has a strictly larger asymptotic
worst-case value. Thus \(p_0\) is the unique pointwise asymptotic
minimizer among these fixed densities.
\end{restatable}

\begin{restatable}[Unbounded Euclidean case]{prop}{propknneucl}
Let $\mathcal M=\mathbb R^m$, let $p$ be a positive nonconstant smooth
density, and retain $q$ and $\mathcal G_n$ as above, with
$\mathcal M=\mathbb R^m$. Then
\[
\lim_{n\to\infty}\sup_{f\in\mathcal G_n}
\int_{\mathbb R^m}
\left(
\Delta f+2\langle\nabla f,\nabla\log p\rangle
\right)^2p^q\,dx
=\infty.
\]
In particular, this holds for every nondegenerate Gaussian density.
\end{restatable}

\section{Proofs: $k$-Nearest Neighbors}
\lemmafour*
\begin{proof}
For each $x\in\mathcal{M}$, we use the exponential map $\exp_x: T_x \mathcal{M} \to \mathcal{M}$ to transform the integral over the geodesic ball $B_r(x)$ to the tangent ball $B^{m}_r(0) \subseteq \mathbb{R}^m$. In Riemannian normal coordinates at $x$, the metric $g$ satisfies the following standard\footnote{See \citep[Corollary 2.9]{gray1974volume} or \citep[Exercise 5.9.43, p. 229]{petersen2006riemannian}. Use the Jacobi identity and the fact that these are normal coordinates (so first derivatives of the metric vanish at $x$) to expand the metric in terms of the Riemannian curvature tensor $R$:
$$g_{kl}(y) = \delta_{kl} - \frac{1}{3}R_{kilj}(x)y^iy^j + O(||y||^3).$$
Then apply the approximation $\text{det}(I+\epsilon A) \approx I + \epsilon \text{tr}(A)$ to obtain, with $h_{kl}=- \frac{1}{3}R_{kilj}(x)y^iy^j$
$$\text{det}\, g \asymp 1+\text{tr}(h) = 1-\frac{1}{3}\delta^{kl}R_{kilj}y^iy^j=1-\frac{1}{3}R^l_{ilj}y^iy^j=1-\frac{1}{3}\text{Ric}_{ij}y^iy^j,$$
since with the sign convention $R(X,Y)Z = \nabla_X\nabla_YZ-\nabla_Y\nabla_XZ-\nabla_{[X,Y]}Z$, the Ricci curvature is a contraction of the Riemannian curvature over the first and third indices, ensuring positive curvature for a sphere. The Taylor approximation $(1-\epsilon)^{1/2} \asymp 1 - \frac{1}{2}\epsilon$
yields
$$\sqrt{\text{det}\, g} \asymp 1-\frac{1}{6}\text{Ric}_{ij}y^iy^j.$$} approximation:
$$\sqrt{\text{det}\, g} = 1 - \frac{1}{6} \text{Ric}_{ij}(x) y^i y^j + O(\|y\|^3), $$
where $\text{Ric}$ is the Ricci tensor and $y^i$ the $i$th coordinate of a point $y$ in an orthonormal basis of $T_x\mathcal{M}$. Thus,
$$\text{dvol}(y) \approx 1 - \frac{1}{6} \text{Ric}_{ij}(x) y^i y^j \mathrm{d}y + O(||y||^3) \mathrm{d}y$$
where we use the notation $\mathrm{d}y = \mathrm{d}y^1\wedge \ldots \wedge \mathrm{d}y^n.$

Let us now obtain the leading terms of the numerator and denominator using Taylor expansions in normal coordinates. Starting with the denominator, we multiply our expansion of $\text{dvol}(y)$ with a Taylor expansion of $p(y)$
$$p(y) = p(0) + \nabla_k p(0) y^k + \frac{1}{2}\nabla_k \nabla_l p(0) y^ky^l + O(||y||^3)$$
 and integrate. We note that by symmetry and linearity, odd terms $\nabla_k p y^k$ integrate to zero over the ball, and distinct coordinates $y^ky^l$ integrate to zero as well, so $\int_{B_r^m(0)} y^k y^l \mathrm{d}y = \delta_{kl}  V_m r^{m+2}/(m+2),$ where $V_m 
 \coloneqq \pi^{m/2}/\Gamma\left(\frac{m}{2} + 1\right)$ is the volume of the unit ball in $\mathbb{R}^m$. 
\begin{align*}
D
&= \int_{B_{r}^m(0)} p(y)\,\mathrm{dvol}(y) \\
&= \int_{B_{r}^m(0)}
\Bigl[
p(0)+\nabla_k p(0)y^k
+\frac12\nabla_k\nabla_l p(0)y^k y^l
+O(\|y\|^3)
\Bigr] \\
&\qquad\qquad\cdot
\Bigl[
1-\frac16\operatorname{Ric}_{ij}(0)y^i y^j
+O(\|y\|^3)
\Bigr]\,\mathrm{d}y \\
&= V_m r_k(x)^m
\left[
p(x)
+\frac{r_k(x)^2}{2(m+2)}\Delta_{\mathcal M}p(x)
-\frac{r_k(x)^2}{6(m+2)}p(x)\operatorname{Scal}(x)
+O(r_k(x)^3)
\right].
\end{align*}
since the symmetry collapses the sum of second-order derivatives into a trace:
$$\nabla_k\nabla_l p(0) \delta^{kl}= \Delta_\mathcal{M} p(x) \qquad  \text{and} \qquad \text{Ric}_{ij}(0)\delta^{ij} = \text{Scal}(x)$$
where $\Delta_{\mathcal{M}}$ is the Laplace-Beltrami operator of $p$ at $x$, and $\mathrm{Scal}(x)$ is the scalar curvature of the manifold at $x$.

We obtain the numerator $N$,
$$V_m r^m \Bigg[ f(x)p(x) + \frac{r_k(x)^2}{2(m+2)} \Delta_{\mathcal{M}}(fp)(x) - \frac{r_k(x)^2}{6(m+2)} f(x)p(x) \text{Scal}(x) \Bigg] + O(r^{m+3}),$$
by identical reasoning using our Taylor expansion of the product between the target and embedding density
\begin{align*}
    f(y)p(y) & \asymp \left(f(0) + \nabla_k f(0)y^k + \frac{1}{2} \nabla_k \nabla_l f(0) y^k y^l\right)\cdot  \left(p(0) + \nabla_k p(0)y^k + \frac{1}{2} \nabla_k \nabla_l p(0) y^k y^l\right).
\end{align*}
We write $$\Delta_{\mathcal{M}}(fp)(x) = f(x)\Delta_{\mathcal{M}}p(x) + 2\langle \nabla f(x), \nabla p(x)\rangle_g + p(x)\Delta_{\mathcal{M}}f(x)$$
and compute the ratio between the numerator and the denominator:
\begin{align*}
\!\widehat{f}_r(x) &= \!\frac{N}{D} \\&= \!\frac{f(x) p(x) \!+\! \frac{r_k(x)^2  \left[ f(x) \Delta_{\mathcal{M}} p(x) + 2 \langle \nabla f(x), \nabla p(x) \rangle_g + p(x) \Delta_{\mathcal{M}} f(x) - \frac{1}{3} f(x) p(x) \, \text{Scal}(x) \right]}{2(m+2)}}{p(x) + \frac{r_k(x)^2}{2(m+2)} \left[ \Delta_{\mathcal{M}} p(x) - \frac{1}{3} p(x) \, \text{Scal}(x) \right]}\! + O(r_k(x)^3) \\
&= \!f(x) + \frac{r_k(x)^2}{2(m+2)} \left( \Delta_{\mathcal{M}} f(x) + 2 \frac{\langle \nabla f(x), \nabla p(x) \rangle_g}{p(x)} \right) + O(r_k(x)^3),
\end{align*}
where the last step used the approximation $\frac{A+\epsilon B}{C+\epsilon D} = \frac{A+\epsilon B}{C}(1+\epsilon \frac{D}{C})^{-1}  \approx \frac{A+\epsilon B}{C}(1-\epsilon \frac{D}{C}) = \frac{A}{C} + \epsilon \frac{BC-AD}{C^2},$ which follows from the Taylor approximation $(1+\epsilon \frac{D}{C})^{-1} \approx 1-\epsilon\frac{D}{C} + O(\epsilon^2)$.

Thus, the bias at $x$ has the expansion $\frac{r_k(x)^2}{2(m+2)} \left( \Delta_{\mathcal{M}} f(x) + 2 \frac{\langle \nabla f(x), \nabla p(x) \rangle_g}{p(x)} \right)$:
\begin{align}
\textnormal{bias}(x;r_k(x)) &= \widehat{f}_r(x) - f(x) \nonumber \\&= \frac{r_k(x)^2}{2(m+2)} \left( \Delta_{\mathcal{M}} f(x) + 2 \frac{\langle \nabla f(x), \nabla p(x) \rangle_g}{p(x)} \right) + O(r_k(x)^3).
\label{eq:bias}
\end{align}
\end{proof}

\theoremseven*

\begin{proof}
Write
\[
V_m=\frac{\pi^{m/2}}{\Gamma(1+m/2)},
\qquad
\beta_n=\frac{k}{n}.
\]
Throughout this proof, $r_k(x)$ denotes the deterministic population
radius defined by
\[
P(B_{r_k(x)}(x))=\beta_n.
\]
We work with fixed $f,p\in C^3(\mathcal M)$, with $p>0$, on a compact
smooth manifold $\mathcal M$. The remainders below are then uniform in
$x\in\mathcal M$.
In geodesic normal coordinates $y$ at $x$, the volume density is
\[
\mathrm{dvol}(y)
=\left(1-\frac16\operatorname{Ric}_{ij}(x)y^iy^j
+O(\|y\|^3)\right)\mathrm dy.
\]
For any $h\in C^3(\mathcal M)$, Taylor expansion and symmetry of the
Euclidean ball give
\begin{align*}
\int_{B_r(x)}h\,\mathrm{dvol}
&=\int_{B_r^m(0)}
\left(h(x)+\nabla_i h(x)y^i
+\frac12\nabla_i\nabla_jh(x)y^iy^j
+O(\|y\|^3)\right)\\
&\qquad\cdot
\left(1-\frac16\operatorname{Ric}_{ij}(x)y^iy^j
+O(\|y\|^3)\right)\mathrm dy\\
&=V_mr^m\left[
h(x)+\frac{r^2}{2(m+2)}\Delta_{\mathcal M}h(x)
-\frac{r^2}{6(m+2)}h(x)\operatorname{Scal}(x)
+O(r^3)\right].
\end{align*}
In particular, taking $h=1$ yields
\[
\operatorname{vol}(B_r(x))
=V_mr^m\left[1-
\frac{\operatorname{Scal}(x)}{6(m+2)}r^2+O(r^3)\right].
\]
With additional smoothness, symmetry improves the last volume
remainder to $O(r^4)$; this improvement is not needed here.
Applying the expansion with $h=p$ gives
\begin{align}
P(B_{r_k(x)}(x))
&=V_mr_k(x)^m\left[
p(x)+\frac{r_k(x)^2}{2(m+2)}
\left(\Delta_{\mathcal M}p(x)
-\frac{p(x)}{3}\operatorname{Scal}(x)\right)
+O(r_k(x)^3)\right].
\label{eq:probmass_ball}
\end{align}
Equivalently, the remainder outside the brackets is
$O(r_k(x)^{m+3})$.
Define the leading radius
\[
r_0(x)=\left(\frac{\beta_n}{V_mp(x)}\right)^{1/m}
=\left(\frac{k\Gamma(1+m/2)}{n\pi^{m/2}p(x)}\right)^{1/m}.
\]
Since $p$ is positive and bounded away from zero on $\mathcal M$,
inverting \eqref{eq:probmass_ball} gives
\begin{align}
r_k(x)
&=r_0(x)
-\frac{r_0(x)^3}{2m(m+2)}
\left(\frac{\Delta_{\mathcal M}p(x)}{p(x)}
-\frac13\operatorname{Scal}(x)\right)
+O(r_0(x)^4).
\label{radius_knn}
\end{align}
Thus $r_k(x)^2=r_0(x)^2+O(r_0(x)^4)$.
By Lemma~\ref{lemma4}, the pointwise bias of the population-radius
estimator is
\[
\operatorname{bias}(x;r_k(x))
=\frac{r_k(x)^2}{2(m+2)}
\left(\Delta_{\mathcal M}f(x)
+2\langle\nabla f(x),\nabla\log p(x)\rangle_g\right)
+O(r_k(x)^3).
\]
Because $r_0(x)=O(\beta_n^{1/m})$ uniformly in $x$, substituting
\eqref{radius_knn} changes only the higher-order terms. Setting
\[
A_{k,n}=
\frac{1}{2(m+2)}
\left(\frac{k\Gamma(1+m/2)}{n\pi^{m/2}}\right)^{2/m},
\]
we obtain
\[
\operatorname{bias}(x;r_k(x))
=A_{k,n}p(x)^{-2/m}
\left(\Delta_{\mathcal M}f(x)
+2\langle\nabla f(x),\nabla\log p(x)\rangle_g\right)
+O(\beta_n^{3/m}).
\]
The leading term is $O(\beta_n^{2/m})$, so its product with the
remainder is $O(\beta_n^{5/m})$. Consequently,
\begin{align}
\operatorname{bias}(x;r_k(x))^2
&=A_{k,n}^2p(x)^{-4/m}
\left(\Delta_{\mathcal M}f(x)
+2\langle\nabla f(x),\nabla\log p(x)\rangle_g\right)^2
+O(\beta_n^{5/m}).
\label{eq:sq_bias}
\end{align}
Integrating against $p(x)\,\mathrm{dvol}(x)$, and using compactness,
gives
\begin{align*}
\operatorname{ISB}(f,p;k,n)
&=A_{k,n}^2\int_{\mathcal M}
\left(\Delta_{\mathcal M}f(x)
+2\langle\nabla f(x),\nabla\log p(x)\rangle_g\right)^2
p(x)^{1-4/m}\,\mathrm{dvol}(x)
+O\!\left(\left(\frac{k}{n}\right)^{5/m}\right).
\end{align*}
\end{proof}

\begin{remark} 
If $f$ and $p$ are $C^4$ and the metric is smooth, the symmetry of the geodesic ball cancels the third-order terms.
The remainder in \eqref{isb_expanded} then improves to $O((k/n)^{6/m})$.
\end{remark}

\propknn*

\begin{proof}
Set $\rho_n=(k/n)^{1/m}$.
Adding a constant to $f$ changes neither its bias nor the
constraints defining $\mathcal G_n$. We may therefore assume
$\int_{\mathcal M}f\,\mathrm{dvol}=0$.
By compactness, $\|f\|_\infty\le C_{\mathcal M}\|\nabla f\|_\infty \le C_{\mathcal M}M_n$.

In the Taylor expansion used to prove Proposition~\ref{theorem7},
the pointwise bias remainder is bounded by $C_pM_n\rho_n^3$,
uniformly in $x\in\mathcal M$ and $f\in\mathcal G_n$.
Moreover, the elliptic estimate
$\|\nabla f\|_2\le C_{\mathcal M}\|\Delta f\|_2$
bounds the $L^2$ norm of the leading bias uniformly on
$\mathcal G_n$ by $C_p\rho_n^2$.
Squaring and integrating thus gives, uniformly for
$f\in\mathcal G_n$,
\[
\mathrm{ISB}(f,p;k,n)
=A(m,k,n)\int_{\mathcal M}
\bigl(\Delta f+
2\langle\nabla f,\nabla\log p\rangle_g\bigr)^2
p^q\,\mathrm{dvol}
+O\!\left(M_n\rho_n^5+M_n^2\rho_n^6\right).
\]
Since $A(m,k,n)\asymp\rho_n^4$ and $M_n\rho_n\to0$,
the remainder is $o(A(m,k,n))$ uniformly on $\mathcal G_n$.

For $p_0$, $\nabla\log p_0=0$, so
\[
\frac1{A(m,k,n)}
\sup_{f\in\mathcal G_n}\ISB(f,p_0;k,n)
\longrightarrow c^2p_0^q.
\]
The upper bound follows from $\|\Delta f\|_2\le c$; equality in the limit is obtained by taking any fixed smooth nonconstant function and scaling it so that $\|\Delta f\|_2=c$.

Fix $p$ and $\varepsilon>0$. Choose a coordinate ball $U$ where $p^q\ge\|p^q\|_\infty-\varepsilon$, a nonzero $\chi\in C_c^\infty(U)$, and set
\[
u_\ell(x)=\chi(x)\cos(\ell x^1),\qquad
f_\ell=\frac{c}{\|\Delta u_\ell\|_2}
\left(u_\ell-\operatorname{vol}(\mathcal M)^{-1}\int_\mathcal M u_\ell\,\dvol\right).
\]
Then $\|\Delta f_\ell\|_2=c$ and $\|\nabla f_\ell\|_2=O(\ell^{-1})$. The reverse triangle inequality therefore yields
\[
\left[
\int_{\mathcal M}
\bigl(\Delta f_\ell+2\langle\nabla f_\ell,\nabla\log p\rangle_g\bigr)^2
p^q\,\dvol
\right]^{1/2}
\ge c\sqrt{\|p^q\|_\infty-\varepsilon}
-\|p^q\|_\infty^{1/2}\,
\|2\nabla\log p\|_\infty\,\|\nabla f_\ell\|_2 .
\]
Each fixed $f_\ell$ belongs to $\mathcal G_n$ for all sufficiently large $n$. 
Therefore, for every fixed $\ell$, the asymptotic expansion gives
\[
\liminf_{n\to\infty}
\frac{1}{A(m,k,n)}
\sup_{f\in\mathcal G_n}\mathrm{ISB}(f,p;k,n)
\ge
\int_{\mathcal M}
\bigl(\Delta f_\ell+
2\langle\nabla f_\ell,\nabla\log p\rangle_g\bigr)^2
p^q\,\mathrm{dvol}.
\]
Letting $\ell\to\infty$ in the preceding lower bound and then $\varepsilon\to 0$ yields
\[
\liminf_{n\to\infty}
\frac{1}{A(m,k,n)}
\sup_{f\in\mathcal G_n}\mathrm{ISB}(f,p;k,n)
\ge c^2\|p^q\|_\infty.
\]
Finally, since the average of $p$ is $p_0$,
\[
\|p^q\|_\infty\ge p_0^q.
\]
The inequality is strict for every nonconstant $p$ when $m\neq4$.

When $m=4$, both sides of the preceding inequality equal one, so the previous argument establishes optimality but does not distinguish the uniform density from nonuniform ones.
Uniqueness can nevertheless be recovered through a separate argument exploiting the embedding $\mathcal M\subset\mathbb R^d$. Let $p$ be nonuniform. We show that its leading worst-case ISB is strictly larger than that of the uniform density.
Let $x_1,\ldots,x_d$ be the coordinate functions of $\mathbb R^d$ restricted to $\mathcal M$, and let $H=(\Delta x_1,\ldots,\Delta x_d)$ denote the mean-curvature vector of $\mathcal M$.
Since $m=4$, we have $p^q=1$, and the leading ISB functional is proportional to
$$Q(f) \coloneq \int_{\mathcal{M}} (\Delta f + 2\left\langle \nabla f, \nabla \log p\right\rangle_g)^2\, \mathrm{dvol}.$$

We show that 
$$
\sup_{||\Delta f||_2\le c}Q(f) \geq c^2\left(1+\frac{4\int_\mathcal{M}||\nabla\log p||_g^2\,\mathrm{dvol}}
{\int_\mathcal{M}||H||^2\,\mathrm{dvol}}\right)> c^2$$
by summing $Q$ over the coordinate functions:
\begin{align*}
\sum_{j=1}^d Q(x_j)
&=\sum_{j=1}^d||\Delta x_j||_2^2
+4\int_\mathcal{M}\left\langle{\textstyle\sum_{j=1}^d}(\Delta x_j)\nabla x_j,\,
\nabla\log p\right\rangle_g \,\mathrm{dvol}
\\&\qquad + 4\sum_{j=1}^d\int_\mathcal{M}\left\langle\nabla x_j,\nabla\log p\right\rangle_g^2\,\mathrm{dvol}.
\end{align*}
The second of the three terms vanishes identically: $H$ is in the normal bundle of $\mathcal{M}$ by the Gauss formula, so its projection onto the tangent space  $\sum_{j=1}^d(\Delta x_j)\nabla x_j$ is zero. The third term also simplifies: $\nabla x_j$ is the tangential projection of the $j$th standard basis vector, so
$\sum_{j=1}^d\left\langle\nabla x_j,X\right\rangle_g^2=||X||_g^2$ pointwise for every tangent field $X$. If we take $X=\nabla\log p$, then our equality becomes
$$ \sum_{j=1}^d \big(Q(x_j)-||\Delta x_j||_2^2\big)
=4\int_\mathcal{M}||\nabla\log p||_g^2\,\mathrm{dvol}>0.$$

Indices $j$ for which $\Delta x_j \equiv 0$ contribute $0$ to both sums; in this case, $x_j$ is then harmonic, and hence constant on the closed connected $\mathcal{M}$, so $\nabla x_j=0$. We can discard them; not all will be discarded, since $\mathcal{M}$ is not a point.
Letting $i$ be a non-discarded index, we apply the mediant inequality and the identity $\sum_{j=1}^d||\Delta x_j||_2^2
=\int_\mathcal{M}||H||_{\mathbb R^d}^2\,\mathrm{dvol}$ and rearrange to obtain 
$$\frac{Q(x_i)}{\left\lVert\Delta x_i\right\rVert_2^2} \geq 
\frac{\sum_{j=1}^d Q(x_j)}{\sum_{j=1}^d \left\lVert\Delta x_j\right\rVert_2^2} = 1 + \frac{4 \int_\mathcal{M} \left\lVert\nabla \log p\right\rVert_g^2 \, \mathrm{dvol}}{\int_\mathcal{M}\left\lVert H\right\rVert_{\mathbb R^d}^2\,\mathrm{dvol}}.$$
Choose $f=c\,x_i/\|\Delta x_i\|_2$.
Then $\|\Delta f\|_2=c$, and this fixed smooth function belongs to $\mathcal G_n$ for all sufficiently large $n$.
The uniform asymptotic expansion, therefore gives
\[
\liminf_{n\to\infty}
\frac{\sup_{f\in\mathcal G_n}\mathrm{ISB}(f,p;k,n)}
     {A(m,k,n)}
\ge
c^2\left(
1+\frac{4\int_{\mathcal M}
    \|\nabla\log p\|_g^2\,\mathrm{dvol}}
   {\int_{\mathcal M}
    \|H\|_{\mathbb R^d}^2\,\mathrm{dvol}}
\right)>c^2.
\]
\end{proof}

\propknneucl*

\begin{proof}
Since $p$ is nonconstant, there exists $x_0\in\mathbb R^m$ such that
$\nabla\log p(x_0)\neq0$. Set
\[
v=\nabla\log p(x_0).
\]
By continuity and positivity of $p$,
\[
\int_{\mathbb R^m}
\langle v,\nabla\log p(x)\rangle^2p(x)^q\,dx>0.
\]
For $R>0$, set $f_R(x)=R\,v^\top x$. Then
\[
\nabla f_R=Rv,\qquad
\Delta f_R=0,\qquad
\|\Delta f_R\|_2=0,\qquad
|f_R|_{C^3}=R\|v\|.
\]
Hence, for every fixed $R$, $f_R\in\mathcal G_n$ for all sufficiently
large $n$, since $M_n\to\infty$. Moreover,
\[
\int_{\mathbb R^m}
\bigl(\Delta f_R+2\langle\nabla f_R,\nabla\log p\rangle\bigr)^2
p^q\,dx
=4R^2\int_{\mathbb R^m}
\langle v,\nabla\log p(x)\rangle^2p(x)^q\,dx.
\]
Therefore,
\[
\liminf_{n\to\infty}\sup_{f\in\mathcal G_n}
\int_{\mathbb R^m}
\bigl(\Delta f+2\langle\nabla f,\nabla\log p\rangle\bigr)^2p^q\,dx
\ge
4R^2\int_{\mathbb R^m}
\langle v,\nabla\log p(x)\rangle^2p(x)^q\,dx
\]
for every $R>0$. Letting $R\to\infty$ proves the claim.

For completeness, if $p=\mathcal N(\mu,\Sigma)$ with $\Sigma\succ0$,
then
\[
\nabla\log p(x)=-\Sigma^{-1}(x-\mu).
\]
Thus, for every $v\ne0$,
$\langle v,\nabla\log p\rangle$ is not identically zero, and
the integral in the assumption is positive (possibly $+\infty$ when
$q\le0$). Hence, every
nondegenerate Gaussian satisfies the assumption.
\end{proof}

\end{document}